\documentclass[lettersize,journal]{IEEEtran}
\usepackage{amsmath,amssymb,amsfonts}
\usepackage{algorithm}
\usepackage{algpseudocode}
\usepackage{array}
\usepackage[caption=false,font=normalsize,labelfont=sf,textfont=sf]{subfig}
\usepackage{textcomp}
\usepackage{stfloats}
\usepackage{url}
\usepackage{verbatim}
\usepackage{graphicx}
\usepackage{lettrine}
\usepackage{epsfig}
\usepackage{cite}
\usepackage{hyperref}
\usepackage{color}
\usepackage{mathtools}
\usepackage{bm}
\usepackage{xcolor}
\usepackage{tikz}
\usetikzlibrary{arrows.meta, positioning, calc}
\usepackage[percent]{overpic}
\usetikzlibrary{positioning, arrows.meta}
\usepackage{xcolor}
\newtheorem{lemma}{\textbf{Lemma}}
\newtheorem{theorem}{\textbf{Theorem}}

\begin{document}

\title{Attitude Estimation Using Inertial and Barometric Measurements}

\author{Méloné Nyoba Tchonkeu$^{1}$, Soulaimane Berkane $^{2}$, \textit{Senior Member, IEEE}, and Tarek Hamel$^{3}$, \textit{Fellow Member, IEEE}

\thanks{$^{1}$M. Nyoba Tchonkeu is with the Department of Computer Science and Engineering, University of Quebec in Outaouais, 
        Gatineau, QC J8X3X7, Canada
        {\tt\small (nyom01@uqo.ca)}}%
\thanks{$^{2}$S. Berkane is with the Department of Computer Science and Engineering, University of Quebec in Outaouais, Gatineau, QC J8X3X7, and also with the Department of Electrical Engineering, Lakehead University,
        Thunder Bay, ON P7B 5E1, Canada
        {\tt\small (soulaimane.berkane@uqo.ca)}}%
\thanks{$^{3}$T. Hamel is with I3S-UniCA-CNRS, University Cote d’Azur and the Insitut Universitaire de France,
        06903 Sophia Antipolis, France
        {\tt\small (thamel@i3s.unice.fr)}}%
}

\maketitle

\begin{abstract}
Accurate and robust attitude estimation is a key challenge for autonomous vehicles, particularly in GNSS-denied conditions and during highly accelerated flight. In such conditions, Inertial Measurement Units (IMUs) alone are insufficient for reliable tilt estimation due to the ambiguity between gravitational and inertial accelerations. Although auxiliary velocity sensors such as GNSS, Pitot tubes, Doppler radar, or Visual Inertial Odometry are commonly used, they may be unavailable, intermittent, or costly.

This paper introduces a barometer-aided attitude estimation architecture that exploits barometric altitude measurements to provide complementary information on the vehicle's vertical motion, thereby enhancing attitude estimation within nonlinear observers on $\mathrm{SO}(3)$. The contributions are twofold. First, we design a deterministic Riccati observer cascaded with a complementary filter, ensuring almost-global asymptotic stability (AGAS) under a uniform observability (UO) condition while preserving the geometric structure of the attitude dynamics. Second, we propose a nonlinear observer evolving on $\mathrm{SO}(3)\times\mathbb{R}^2$, which integrates IMU measurements as inputs and barometer and magnetometer measurements as outputs within a unified framework, guaranteeing local exponential stability (LES) under relaxed uniform observability conditions.

The proposed approaches are validated using both simulated and real flight data. The results demonstrate that barometer-aided estimation provides a lightweight, reliable, and effective complementary sensing modality for attitude estimation in minimal-sensing configurations, offering a practical alternative when conventional velocity measurements are unavailable or degraded.
\end{abstract}


\section{Introduction}
\lettrine{A}ccurate attitude estimation for autonomous vehicles operating in GNSS-denied conditions or undergoing significant linear accelerations remains a central challenge in inertial navigation. The Inertial Measurement Unit (IMU), which provides angular velocity and specific acceleration measurements, constitutes the primary sensing backbone of most attitude estimation systems. In many estimation frameworks (\textit{e.g.,}~\cite{mahony2008nonlinear,hua2013implementation,grip2011attitude}), particularly those targeting low-cost UAVs and robotic platforms, accelerometer measurements are used as proxies for gravity under the assumption of negligible linear accelerations. However, accelerometers do not measure gravity directly; rather, they measure specific acceleration, which includes all non-gravitational accelerations acting on the vehicle. This approximation is therefore valid only under quasi-static conditions, where gravity dominates the measured acceleration. During aggressive maneuvers, legged locomotion, or in the presence of strong wind disturbances, the accelerometer signal can significantly deviate from the gravity direction, leading to ambiguity in tilt estimation. To overcome this limitation, modern attitude estimation architectures increasingly incorporate complementary sensing modalities to improve estimation and robustness.

Several authors have addressed this issue by incorporating linear velocity information, either in the body-fixed frame or in the inertial frame. This class of approaches, commonly referred to as the \emph{velocity-aided attitude} (VAA) problem, has attracted considerable attention in recent years \cite{hua2010attitude,roberts2011attitude,berkane2017attitude}. Early solutions to the body-frame VAA problem relied on linearisation, e.g., \cite{2008_bonnabel_SymmetryPreservingObservers,2008_martin_InvariantObserverEarthVelocityAided,mourikis2007multi,bloesch2015robust}, whereas later approaches adopted more constructive designs and, in some cases, provided guarantees of almost-global asymptotic stability \cite{2013_troni_PreliminaryExperimentalEvaluation,2016_hua_StabilityAnalysisVelocityaided,hansen2017nonlinear,wang2021nonlinear,benallegue2023velocity,Pieter2023,bouazza2025observer}.
Although robust and theoretically grounded, these architectures generally assume full measurement of either body or inertial velocity, which limits their practical deployment when only partial velocity information is available or when measurements are unreliable. This limitation has received little theoretical attention. A notable exception is the work of Oliveira \textit{et al.} \cite{oliveira2024pitot} that addresses tilt and air velocity estimation for fixed-wing UAVs in GNSS-denied conditions by exploiting only a component of the air velocity vector in the body fixed frame provided by a \textit{single-axis Pitot tube}  and IMU data within a Riccati observer framework on $\mathrm{SO}(3) \times \mathbb{R}^3$ and guarantees local asymptotic convergence. Due to the limited sensing configuration, the system exhibits a structural yaw ambiguity and requires sufficiently rich persistent excitation to ensure observability. To alleviate these limitations, subsequent approaches incorporate additional information, such as zero side-slip angle pseudo-measurements and magnetometer measurement within cascade structures in which a deterministic observer is combined with a complementary filter on $\mathrm{SO}(3)$, ensuring almost global asymptotic stability\cite{tchonkeu2026pitot}. While relevant, these approaches remain sensitive to uncertainties in Pitot tube measurements.

Unlike Pitot tubes, barometers are inexpensive, lightweight, widely available, and largely insensitive to airflow disturbances. Yet, their potential as attitude-aiding sensors remains largely unexplored. The key observation underlying this work is that barometric altitude measurements provide information about the vehicle's motion along the gravity direction and can therefore be exploited to enhance attitude estimation even in the absence of conventional velocity measurements.

Building on this observation, we develop two complementary barometer-aided attitude estimation architectures within geometric observer frameworks. The first architecture adopts a cascade design in which a deterministic Riccati observer estimates the vertical motion and tilt information from inertial and barometric measurements, while a nonlinear observer on $\mathrm{SO}(3)$ exploits magnetometer measurements to recover the full attitude and guarantee almost-global asymptotic stability. The second architecture exploits the geometric Riccati observer framework of \cite{hamel2017riccati} by directly integrating IMU, barometric, and magnetometer measurements within a unified reduced-order observer evolving on $\mathrm{SO}(3)\times\mathbb{R}^2$. The resulting design guarantees local exponential stability under weaker observability requirements.

Beyond the observer design itself, the paper provides a theoretical characterization of the observability properties of barometer-aided attitude estimation and establishes rigorous convergence guarantees for both architectures. The resulting framework reveals an interesting tradeoff between almost-global convergence and local exponential convergence, while offering a practical and lightweight alternative to conventional velocity-aided approaches. Extensive simulations and real-flight experiments corroborate the theoretical findings and demonstrate the effectiveness of barometric sensing as a complementary source of attitude information in GNSS-degraded environments.

This paper extends the preliminary results of \cite{tchonkeu2025barometer} by introducing a complementary barometer-aided attitude observer with local exponential stability (LES) guarantees and providing a unified comparison with the previously proposed almost-global asymptotic stability (AGAS) observer.

The remainder of the paper is structured as follows. Section~\ref{sec:prelims} reviews the required preliminary material and introduces uniform observability definitions. Section~\ref{sec:prob_desc} describes the problem and the objective of this work. Section \ref{sec:obs_design} presents the proposed one-stage and two-stage observer designs, establishes their observability and stability properties, and details their discrete-time implementation algorithms. Section \ref{sec:sim_exp} reports and discusses the simulation and experimental results. Finally, concluding remarks are provided in Section \ref{sec:concl}.

\section{Preliminary Material} \label{sec:prelims}
We denote by \( \mathbb{R} \) and \( \mathbb{R}_+ \) the sets of real and nonnegative real numbers, respectively. The \( n \)-dimensional Euclidean space is denoted by \( \mathbb{R}^n \). The Euclidean inner product of two vectors \( a, b \in \mathbb{R}^n \) is defined as \( \langle a, b \rangle = a^\top b \). The associated Euclidean norm of a vector \( a \in \mathbb{R}^n \) is \( |a| = \sqrt{a^\top a} \). Furthermore, we denote by \( \mathbb{R}^{m \times n} \) the set of real \( m \times n \) matrices. The set of \( n \times n \) positive definite matrices is denoted by \( \mathcal{S}^+(n) \), and the identity matrix is denoted by \( I_n \in \mathbb{R}^{n \times n} \). Given two matrices \( X, Y \in \mathbb{R}^{m \times n} \), the Euclidean matrix inner product is defined as \( \langle X, Y \rangle = \mathrm{tr}(X^\top Y) \), and the Frobenius norm of \( X \in \mathbb{R}^{n \times n} \) is given by \( \|X\| = \sqrt{\langle X, X \rangle} \). 
The unit sphere \( \mathbb{S}^{n-1} := \{ \eta \in \mathbb{R}^n \mid |\eta| = 1 \} \subset \mathbb{R}^n \) denotes the set of unit  vectors and forms a smooth submanifold of \( \mathbb{R}^n \). 

For any vector $\eta\in\mathbb{S}^{n-1}$, let
\[
\Pi_{\eta}:=I_n-\eta\eta^\top
\]
denote the orthogonal projection onto the plane normal to $\eta$.
For any vector $\gamma\in\mathbb{R}^n$, define
\[
\bar{\Pi}_{\gamma}:=|\gamma|^2I_n-\gamma\gamma^\top.
\]
Notice that
\(
\bar{\Pi}_{\gamma}=|\gamma|^2\Pi_{\gamma/|\gamma|}
\)
whenever $\gamma\neq0$, while $\bar{\Pi}_{0}=0_{n\times n}$.
By $\mathrm{diag(\cdot)}$, we denote the block diagonal matrix.
We denote by \( \mathcal{I}\) the inertial frame and by \( \mathcal{B}\) the body-fixed frame rigidly attached to the vehicle at the IMU location. The inertial frame is chosen as the North-East-Down (NED). 
Let $(\mathrm{e}_1,\mathrm{e}_2,\mathrm{e}_3)$ and $(\mathrm{e}_1^{\mathcal{B}},\mathrm{e}_2^{\mathcal{B}},\mathrm{e}_3^{\mathcal{B}})$ denote the canonical basis vectors of  \( \mathcal{I}\) and  \( \mathcal{B}\), respectively. The vectors \( \mathrm{e}_1 = \begin{bmatrix}1  & 0 & 0 \end{bmatrix}^\top \in \mathbb{S}^2 \) and \( \mathrm{e}_2 = \begin{bmatrix}0  & 1 & 0 \end{bmatrix}^\top \in \mathbb{S}^2 \) span the horizontal plane, while \( \mathrm{e}_3 = \begin{bmatrix}0  & 0 & 1 \end{bmatrix}^\top \in \mathbb{S}^2 \) is aligned with the gravity direction. Similarly, $e_1^{\mathcal{B}},e_2^{\mathcal{B}}$ span the horizontal plane of \( \mathcal{B}\), and $e_3^{\mathcal{B}}$ denotes its vertical axis.

The special orthogonal group of 3D rotations is denoted by
$
\mathrm{SO}(3) := \{ R \in \mathbb{R}^{3 \times 3} \mid RR^\top = R^\top R = I_3,\ \det(R) = 1 \}.
$
The Lie algebra of $\mathrm{SO}(3)$ is 
$
\mathfrak{so}(3) := \{\, \Omega \in \mathbb{R}^{3\times 3} \mid \Omega^\top = -\Omega \,\},
$
isomorphic to $\mathbb{R}^3$ via the skew-symmetric operator 
$(\cdot)^\times : \mathbb{R}^3 \to \mathfrak{so}(3)$, defined such that 
$
x \times y = x^\times y$ for all $x,y \in \mathbb{R}^3.$
The exponential map \( \exp : \mathfrak{so}(3) \rightarrow \mathrm{SO}(3) \) defines a local diffeomorphism from a neighborhood of \( 0 \in \mathfrak{so}(3) \) to a neighborhood of \( I_3 \in \mathrm{SO}(3) \). This induces the mapping of \( \exp \circ (\cdot)^\times : \mathbb{R}^3 \rightarrow \mathrm{SO}(3) \), defined via Rodrigues' formula \cite{Ma2004rodriguesformula}:
\begin{equation}
\exp([\theta]^{^\times}) 
= I_3 - \frac{\sin(|\theta|)}{|\theta|}[\theta]^{^\times}
+ \frac{1-\cos(|\theta|)}{|\theta|^2}([\theta]^{^\times})^2.
\label{eq:rodriguesformula}
\end{equation}

Consider the linear time-varying (LTV) system given by
\begin{equation}
\begin{cases}
\dot{x} = A(t)x + B(t)u, \\
y = C(t)x,
\end{cases}
\label{eq:LTV_system}
\end{equation}
with state \( x \in \mathbb{R}^n \), input \( u \in \mathbb{R}^\ell \), and output \( y \in \mathbb{R}^m \). The matrix-valued functions \( A(t) \), \( B(t) \), and \( C(t) \) are assumed to be continuous and bounded. 
Let the \emph{observability Gramian} of the system be denoted by
\begin{equation}
W(t, t+\tau) := \frac{1}{\tau} \int_t^{t+\tau} 
\Phi^\top(s, t) \, C^\top(s) \, C(s) \, \Phi(s, t) \, ds, \label{eq:W}
\end{equation}
where $\tau > 0$ and \( \Phi(s, t) \) is the state transition matrix such that
\begin{equation}
\frac{d}{dt} \Phi(s, t) = A(t)\Phi(s, t), 
\quad \Phi(t, t) = I_n, \quad \forall s \geq t. \label{eq:transition_matrix} 
\end{equation}
By definition from~\cite{Besancon2007}, the system ~\eqref{eq:LTV_system} or pair \( (A(t), C(t)) \) is \emph{uniformly observable} if there exist constants 
\( \delta, \mu > 0 \) such that, for all \( t \geq 0 \),
\begin{equation}
W(t, t+\delta) \geq \mu I_n. \label{eq:observability_gramian}
\end{equation}

\section{Problem Description}\label{sec:prob_desc}
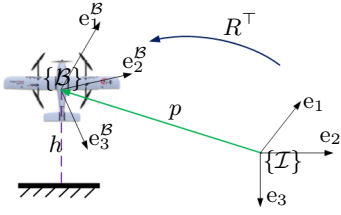
\begin{figure}[ht]
  \centering
  \begin{overpic}[width=0.5\linewidth]{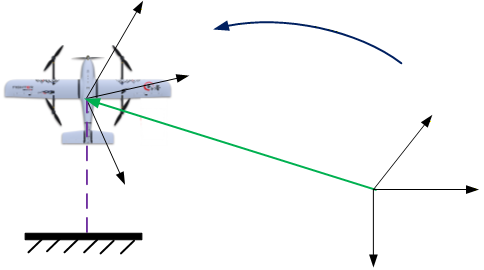}
    \put(12,37){\small \(\left\{\mathcal{B}\right\}\)}
    \put(23,55){\small \(\mathrm{e}_1^{\mathcal{B}}\)}
    \put(36,42){\small \(\mathrm{e}_2^{\mathcal{B}}\)}
    \put(26,19){\small \(\mathrm{e}_3^{\mathcal{B}}\)}
    \put(78,12){\small \(\left\{\mathcal{I}\right\}\)}
    \put(90,30){\small \(\mathrm{e}_1\)}
    \put(95,20){\small \(\mathrm{e}_2\)}
    \put(79,2){\small \(\mathrm{e}_3\)}
    \put(50,28){\small \(p\)}
    \put(66,52){\small \(R^\top\)}
    \put(14,17){\small \(h\)}
  \end{overpic}
  \caption{We assumed a vehicle equipped with an IMU and a barometer measuring the height $h$.}
  \label{fig:veh_baro_imu}
\end{figure}

Let \( R \in \mathrm{SO}(3) \) denote the rotation matrix describing the orientation of the
body-attached frame \( \mathcal{B} \) with respect to the inertial frame \( \mathcal{I} \). The position and linear velocity of the rigid body, expressed in the inertial frame \( \mathcal{I} \), are denoted by \( p \in \mathbb{R}^3 \) and \( v \in \mathbb{R}^3 \), respectively.
The IMU components provide body-fixed measurements of the angular velocity \( \mathrm{\omega} \in \mathbb{R}^3 \) and the linear specific acceleration \( \mathrm{a} \in \mathbb{R}^3 \). The vehicle dynamics are described by the second-order translational model and the first-order attitude kinematics
\begin{align}
\dot{p} &= v, \label{eq:position_dyn} \\
\dot{v} &= R\mathrm{a} + \mathrm{g}, \label{eq:velocity_dyn} \\
\dot{R} &= R\mathrm{\omega}^\times,
\label{eq:attitude_dyn}
\end{align}
where \( \mathrm{g} = g \mathrm{e}_3 \in \mathbb{R}^3 \) is the gravitational acceleration expressed in the inertial frame and \( g \approx 9.81\, \text{m/s}^2 \) is the gravity constant. To enhance observability of the vertical motion, we define the altitude as 
$h := \mathrm{e}_3^\top p$. The barometric sensor output is modeled as
\begin{equation}
y_b = h + n_b,
\label{eq:baro_measurement}
\end{equation}
where $n_b \sim \mathcal{N}(0,\sigma_b^2)$ is zero-mean noise. 
This scalar measurement provides partial information on the inertial position and, through its time derivatives, can also contribute to constraining the gravity direction for tilt estimation, particularly when GNSS velocity data are unavailable. 
In addition, we assume available a magnetometer that provides measurements of the Earth's magnetic field,
\begin{equation}
y_m := \mathrm{m}_{\mathcal{B}} = R^\top \mathrm{m}_{\mathcal{I}}+n_m,
\label{eq:mag_measurement}
\end{equation}
where $\mathrm{m}_{\mathcal{I}} \in \mathbb{S}^2$ is the known magnetic field vector expressed in the inertial frame and $n_m \sim \mathcal{N}(0,\sigma_m^2) \in \mathbb{R}^3$ is zero-mean noise.
In summary, the vehicle is modeled by \eqref{eq:position_dyn}--\eqref{eq:attitude_dyn} and equipped with an IMU, a barometer, and a magnetometer. The objective is to design nonlinear observers that estimate the attitude $R \in \mathrm{SO}(3)$ by exploiting the complementary information provided by barometric altitude measurements.
\section{Proposed Observer Design}\label{sec:obs_design}
This section presents two barometer-aided attitude estimation architectures.
The superscripts $(G)$ and $(L)$ denote quantities associated with the
two-stage observer, which guarantees AGAS, and the unified observer, which guarantees LES, respectively.
The observer designs, together with their observability conditions,
stability analyses, and discrete-time implementations, are presented in
the following subsections.
\subsection{Two-Stage Observer Architecture}
This section presents a two-stage observer architecture for attitude estimation. The design exploits a reduced-order Riccati observer for vertical dynamics and tilt estimation, followed by a nonlinear observer on \( \mathrm{SO}(3) \) to estimate full orientation.

\subsubsection{Tilt and Barometric Altitude Estimation}
We define the tilt (reduced attitude) as the gravity direction expressed in the body frame~\cite{benallegue2023velocity}:
\begin{equation}
z := R^\top \mathrm{e}_3 \in \mathbb{S}^2,
\end{equation}
which evolves according to
\begin{equation}
\dot{z} = -\mathrm{\omega}^\times z.
\label{eq:z_dot}
\end{equation}
Defining the altitude by
\[
h := \mathrm{e}_3^\top p,
\]
its first time derivative corresponds to the vertical velocity
\[
v_h := \dot{h}.
\]
Differentiating once more and using \eqref{eq:velocity_dyn} yields
\begin{equation}
\dot{v}_h
=
g+\mathrm{a}^\top z.
\label{eq:h_ddot}
\end{equation}
  
Together, \eqref{eq:z_dot}--\eqref{eq:h_ddot} and the barometer measurement 
\eqref{eq:baro_measurement} can be written as the
linear time-varying (LTV) system.
\begin{equation}
\begin{cases}
\dot{x}^{G} = A^{G}(t)x^G + B^{G}g, \\
y^G = C^{G}x^{G} + n_b,
\end{cases}
\label{eq:ltv_G}
\end{equation}
where $x^{G} := \begin{bmatrix} h & v_h & z^\top \end{bmatrix}^\top \in \mathbb{R}^5,$
\begin{equation}
A^{G}(t) = 
\begin{bmatrix}
0 & 1 & 0_{1 \times 3} \\
0 & 0 & \mathrm{a}^\top \\
0_{3 \times 1} & 0_{3 \times 1} & -\mathrm{\omega}^\times
\end{bmatrix}, \quad 
B^{G} = \begin{bmatrix} 0 \\ 1 \\ 0_{3 \times 1} \end{bmatrix}, \label{eq:state_command_matrix_G}
\end{equation}
and:
\begin{equation}
C^{G} = \begin{bmatrix} 1 & 0 & 0_{1 \times 3} \end{bmatrix}. \label{eq:output_matrix_G}
\end{equation}

To estimate the state vector $x^{G}$ independently of the attitude $R$, we decouple its dynamics from the attitude estimation by designing a deterministic Riccati observer. The observer is given by
\begin{equation}
\dot{\hat{x}}^{G} = A^{G}(t) \hat{x}^{G} + B^{G} g + K^{G}(t)\left(y^G - C^{G} \hat{x}^{G} \right), \; \hat{x}^{G}(0)=\hat{x}^{G}_0\label{eq:riccati_observer_G}
\end{equation}
where $\hat x^{G} := \begin{bmatrix} \hat h^G & \hat v_h^G & \hat z^\top \end{bmatrix}^\top \in \mathbb{R}^5$ is the estimate of $x^G$, and $K^{G}(t) = P^{G}(t) (C^{G})^\top (Q^G)^{-1},$ with \(P^G(t)\) solution to the Continuous-time Riccati Equation (CRE): 
\begin{equation}
\begin{aligned}
\dot{P}^G &= 
A^G(t) P^G + P^G A^G(t)^\top\\
& - P^G (C^G)^\top (Q^G)^{-1} C^G P^G + S^G,\; P^G(0)>0
\end{aligned}
\label{eq:creG}
\end{equation}
The CRE is parameterized by two symmetric positive definite weighting matrices, $Q^G>0$ and $S^G>0$, associated with the barometer measurement and process uncertainties, respectively. Its well-posedness, and hence the stability of the observer, is ensured under a uniform observability
(UO) condition. Under this condition, the CRE admits a unique, bounded, and positive-definite solution \(P^G(t)\) for all \(t\ge0\), and the observer error
\[
\tilde{x}^G:=x^G-\hat{x}^G
\]
converges exponentially to zero, with a convergence rate that depends
on the weighting matrices \(Q^G\) and \(S^G\)~\cite{Hamel2017PositionMeasurements}.
To establish the uniform observability (UO) condition, we derive a
convenient expression for the observability Gramian.
Partition the system matrix in~\eqref{eq:state_command_matrix_G} as
\begin{equation}
A^G(t) =\begin{bmatrix}A^G_{11}(t) & A^G_{12}(t) \\
0_{3 \times 2} & -\mathrm{\omega}^\times(t)
\end{bmatrix}, \label{eq:state_matrix_G}
\end{equation}
where \[ A^G_{11}(t) = \begin{bmatrix} 0 & 1 \\ 0 & 0 \end{bmatrix}, \quad
A^G_{12}(t) = \begin{bmatrix} 0_{3 \times 1} & \mathrm{a}(t) \end{bmatrix}^{\top},\]
and write the output matrix  ~\eqref{eq:output_matrix_G} as:
\[
C^G = \begin{bmatrix} C^G_1 & 0_{1 \times 3} \end{bmatrix},
\quad C^G_1 = [\,1\ 0\,].
\]
Since \(A^G(t)\) in \eqref{eq:state_matrix_G} is block upper triangular, the corresponding state
transition matrix admits the decomposition
\begin{equation}
\Phi^G(t,\tau) =
\begin{bmatrix} 
\phi_{11}(t,\tau) & \phi_{12}(t,\tau) \\
0_{3 \times 2} & \phi_{22}(t,\tau)
\end{bmatrix}, \label{eq:state_trans_matrix_G}
\end{equation}
with $\phi_{11}\in\mathbb{R}^{2\times2}, \phi_{12}\in\mathbb{R}^{2\times3},
\phi_{22}\in\mathbb{R}^{3\times3}.$
From~\eqref{eq:transition_matrix} and~\eqref{eq:state_trans_matrix_G}, the derivative of the transition matrix is given by
\begin{equation}
\begin{aligned}
\frac{d}{dt} \Phi^G(t, \tau)
&=\begin{bmatrix}
A^G_{11}\phi_{11} & A^G_{11}\phi_{12} + A^G_{12}\phi_{22} \\
0 & -\mathrm{\omega}^\times\phi_{22}
\end{bmatrix}, \Phi^{G}(\tau,\tau) = I_5, \label{eq:deriv_trans_matrix_G}
\end{aligned}
\end{equation}
with \(\phi_{11}(\tau,\tau) = I_2\), \(\phi_{22}(\tau,\tau) = I_3\), and \(\phi_{12}(\tau,\tau) = 0_{2\times3}\).
Since \(A^G_{11}\) is a constant matrix, we have from~\eqref{eq:deriv_trans_matrix_G} that 
\begin{equation}
\phi_{11}(t,\tau) = \exp\left(A^G_{11}(t-\tau)\right) = \begin{bmatrix} 1 & (t-\tau) \\ 0 & 1 \end{bmatrix}\label{eq:phi_11}
\end{equation}
In view of \eqref{eq:output_matrix_G}, \eqref{eq:state_trans_matrix_G}, and \eqref{eq:W}, the Gramian of the LTV system ~\eqref{eq:ltv_G} is given by
\begin{equation}
\begin{aligned}
&W^G(t,t+\tau) =\\
&\frac{1}{\tau}\int_{t}^{t + \tau} \begin{bmatrix} \phi_{11}(s, t)^\top \\
\phi_{12}(s, t)^\top\end{bmatrix} (C^G_1)^\top C^G_1 \begin{bmatrix} \phi_{11}(s,t)&\phi_{12}(s,t)
\end{bmatrix}ds.
\end{aligned}
\label{eq:gramian_G}
\end{equation}

\begin{lemma} \label{Lemma1}
\textit{Assume that the body-frame specific acceleration $\mathrm{a}(t)$ and
the angular velocity $\mathrm{\omega}(t)$ are continuous and uniformly
bounded. Furthermore, assume that the inertial-frame acceleration
$a_{\mathcal I}(t):=R(t)\mathrm{a}(t)$ is persistently exciting (PE),
that is, there exist constants $\bar{\delta},\bar{\mu}>0$ such that,
for all $t\ge0$,
\begin{equation}
\frac{1}{\bar{\delta}}
\int_t^{t+\bar{\delta}}
a_{\mathcal I}(s)a_{\mathcal I}(s)^\top\,ds
\ge
\bar{\mu}I_3.
\label{eq:PE_Condition_G}
\end{equation}
Then the pair $(A^G(t),C^G)$ is uniformly observable. Consequently,
the CRE~\eqref{eq:creG} admits a unique bounded positive-definite
solution, and the equilibrium
$\tilde{x}^G=0_{5\times1}$ is globally exponentially stable (GES).}
\end{lemma}

\begin{IEEEproof}
See Appendix~\ref{appendix_A}.
\end{IEEEproof}
Lemma~\ref{Lemma1} requires the inertial acceleration to be persistently
exciting. In practice, this condition is satisfied when the vehicle
undergoes sufficiently rich translational and rotational motions.
\subsubsection{Full Attitude Estimation on $\mathrm{SO}(3)$}
Once the tilt estimate \(\hat{z}\) is available, an additional known
reference direction is required to resolve the remaining heading
ambiguity; see~\cite{Bryne2017}. To this end, we exploit the
magnetometer measurements \(\mathrm{m}_{\mathcal B}\) defined in
\eqref{eq:mag_measurement}. Let \(\hat{R}^G\in\mathrm{SO}(3)\) denote
the estimate of the attitude \(R\). The nonlinear attitude observer is
given by
\begin{equation}
\dot{\hat{R}}^G
=
\hat{R}^G\omega^\times
-
\sigma_R^\times\hat{R}^G,
\qquad
\hat{R}^G(0)=\hat{R}^G_0\in\mathrm{SO}(3),
\label{eq:attitude_observer_G}
\end{equation}
where the innovation term \(\sigma_R\in\mathbb{R}^3\) is defined as
\begin{equation}
\sigma_R = k_z (\mathrm{e}_3 \times \hat{R}^G \hat{z}) + k_m (\bar{\mathrm{m}}_{\mathcal{I}} \times \hat{R} \bar{\mathrm{m}}_{\mathcal{B}}),
\label{eq:attitude_correction_G}
\end{equation}
with \( \bar{\mathrm{m}}_{\mathcal{I}} = \Pi_{\mathrm{e}_3} \mathrm{m}_{\mathcal{I}} \), \( \bar{\mathrm{m}}_{\mathcal{B}} = \bar{\Pi}_{\hat z} \mathrm{m}_{\mathcal{B}} \), 
\( k_z > 0 \), and \( k_m \geq 0 \).
The use of $\bar{\Pi}_{\hat z}$ prevents singularities if $\hat z$ vanishes, and the projected magnetometer vectors ensure that yaw estimation is decoupled from roll and pitch; see \cite{hua2013implementation}. The overall architecture is illustrated in figure~\ref{fig:observer_block_diagram_G}.

\begin{figure}[ht]
\centering
\begin{tikzpicture}[
    font=\small,
    >=Latex,
    node distance=1.2cm,
    sensor/.style={
        draw, thick, rounded corners=2mm,
        fill=gray!12,
        minimum width=2.2cm,
        minimum height=0.7cm,
        align=center
    },
    block/.style={
        draw, thick, rounded corners=2mm,
        minimum width=2.8cm,
        minimum height=1.1cm,
        align=center
    },
    riccati/.style={block, fill=blue!12},
    attitude/.style={block, fill=green!12},
    arrow/.style={->, thick},
    lab/.style={font=\scriptsize, fill=white, inner sep=1.5pt}
]

\node[riccati]  (ric) at (0,  1.4)
{\textbf{Riccati Observer}\\[-1mm]\scriptsize \eqref{eq:riccati_observer_G}};

\node[attitude] (att) at (0, -1.7)
{\textbf{$\mathrm{SO}(3)$ Observer}\\[-1mm]\scriptsize \eqref{eq:attitude_observer_G}};

\node[sensor] (baro) at (-5.2,  2.2) {Barometer};
\node[sensor] (imu)  at (-5.2,  0.25) {IMU\\[-1mm]\scriptsize acc. + gyro};
\node[sensor] (mag)  at (-5.2, -1.35) {Magnetometer};
\node[sensor] (ref)  at (-5.2, -2.65) {Reference field\\[-1mm]\scriptsize $m_{\mathcal I}$};

\node (hhat) at (2.4,  1.6){}; 
\node (vhhat) at (2.4,  1){}; 
\node (Rhat) at (2.5, -1.7){}; 

\coordinate (ric_h) at ($(ric.west)+(0,  0.38)$);
\coordinate (ric_a) at ($(ric.west)+(0, 0.0)$);
\coordinate (ric_w) at ($(ric.west)+(0, -0.38)$);

\coordinate (att_z)  at ($(att.west)+(0,  0.38)$);
\coordinate (att_w)  at ($(att.west)+(0,  0.1)$);
\coordinate (att_mB) at ($(att.west)+(0, -0.17)$);
\coordinate (att_mI) at ($(att.west)+(0, -0.38)$);

\coordinate (imu_split_a) at ($(imu.east)+(0,  0.18)$);
\coordinate (imu_split_w) at ($(imu.east)+(0,  -0.18)$);
\coordinate (w_fork) at ($(imu_split_w)+(0.9,  0)$);

\draw[arrow] (baro.east) -- ++(1.1,0) |- 
    node[lab, near start, above] {$h$} (ric_h);
    
\draw[arrow] (imu_split_a) -- ++(0.5,0.0) |- 
    node[lab, pos=0.3, above] {$a$} (ric_a);
    
\draw[-,thick](imu_split_w)--
    node[lab,above]{$\omega$}(w_fork);
\draw[arrow,thick] (w_fork) |- (ric_w);
\draw[arrow,thick] (w_fork) |- (att_w);
\fill(w_fork) circle(1.5pt);

\draw[arrow] (ric.east |- hhat) -- 
    node[lab, above] {$\hat h^{G}$} (hhat); 

\draw[arrow] (ric.east |- vhhat) -- 
    node[lab, above] {$\hat v_h^{G}$}(vhhat); 

\draw[arrow] (ric.south) -- ++(0,-0.55)
    -- ++(-2.3,0) |-
    node[lab, pos=0.75, above] {$\hat z$} (att_z);


\draw[arrow] (mag.east) -- ++(0.5,0) |- 
    node[lab, near start, above] {$m_{\mathcal B}$} (att_mB);

\draw[arrow] (ref.east) -- ++(1.1,0) |- 
    node[lab, near start, below] {$m_{\mathcal I}$} (att_mI);

\draw[arrow] (att.east) -- 
    node[lab, above] {$\hat R^{G}$} (Rhat.west);

\end{tikzpicture}

\caption{Two-stage observer architecture. A Riccati observer combines IMU and barometric measurements to estimate the tilt $\hat z$, the height $\hat h^{G}$, and the vertical velocity $\hat v_{h}^{G}$. Tilt estimate is then fused with gyroscope and magnetometer measurements in an $\mathrm{SO}(3)$ observer to recover the gravity-referenced attitude estimate $\hat R^{G}$.}
\label{fig:observer_block_diagram_G}
\end{figure}
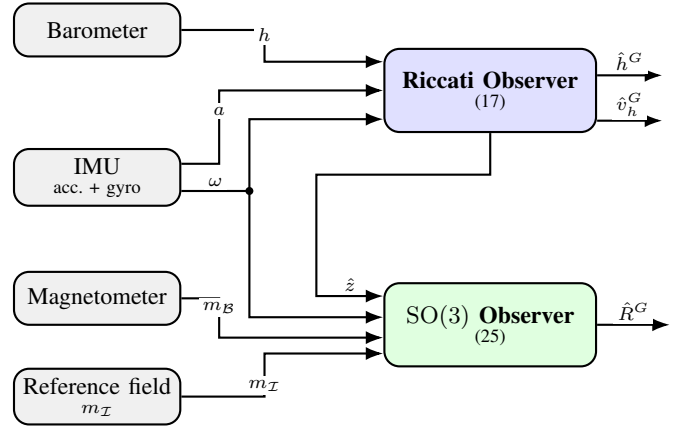

\begin{theorem}\label{Theorem1}
Consider the Riccati observer~\eqref{eq:riccati_observer_G} and the
attitude observer~\eqref{eq:attitude_observer_G} with innovation term
\eqref{eq:attitude_correction_G}. Assume that the pair
\((A^G(t),C^G)\) is uniformly observable, as characterized in
Lemma~\ref{Lemma1}. Then the estimation errors
\[
\tilde{R}^G=R (\hat R^G)^\top,
\qquad
\tilde{x}^G=x^G-\hat{x}^G,
\]
converge asymptotically to the equilibrium set
\[
\mathcal{E}=\mathcal{E}_s\cup\mathcal{E}_u,
\]
where
\[
\mathcal{E}_s=\{(I_3,0_{5\times1})\},
\]
and
\[
\mathcal{E}_u=
\left\{
\left(U\Lambda U^\top,0_{5\times1}\right)
\,\middle|\,
\Lambda=\operatorname{diag}(1,-1,-1),\;
U\in\mathrm{SO}(3)
\right\}.
\]
Moreover, the equilibrium set \(\mathcal{E}_u\) is unstable, whereas
the equilibrium \(\mathcal{E}_s\) is almost globally asymptotically
stable (AGAS).
\end{theorem}

\begin{IEEEproof}
See Appendix~\ref{appendix_B}.
\end{IEEEproof}

Theorem~\ref{Theorem1} establishes almost-global asymptotic stability
(AGAS) of the proposed two-stage observer architecture. A key feature
of the design is the separation between the Riccati observer and the
nonlinear attitude observer. In particular, the estimation error
\(\tilde{x}^G\) converges globally and exponentially to zero
independently of the attitude error dynamics, provided that the pair
\((A^G(t),C^G)\) is uniformly observable. This condition is well posed
since the system matrix \(A^G(t)\) in~\eqref{eq:state_matrix_G} depends
only on the IMU measurements \(\omega(t)\) and \(a(t)\), which are
assumed to be continuous and uniformly bounded. Finally, the AGAS
property is intrinsic to attitude estimation on
\(\mathrm{SO}(3)\): the topology of the rotation group precludes the
existence of a globally asymptotically stable continuous-time observer
\cite{koditschek1988application}.
\subsubsection{Discrete-Time Implementation}\label{sec:implementation_G}
The proposed observer is implemented at the IMU sampling period~$T$. 
Over each interval $[t_k,t_{k+1})$, we assume the measured acceleration 
$\mathrm{a}_k$ and angular velocity $\mathrm{\omega}_k$ are constant 
(see~\cite{Bryne2017}). Let $\Omega_k \doteq \mathrm{\omega}_k^\times$ 
and $\theta_k \doteq |\mathrm{\omega}_k|T$. Approximating the discrete process noise by $S^G_{d,k}\approx S^G_k T$, the transition block associated  with the tilt dynamics satisfies 
$\dot \phi_{22}(t)=-\Omega_k \phi_{22}(t)$ with $\phi_{22}(0)=I_3$, 
which integrates to the incremental rotation
\begin{equation}
\phi_{22,k} 
= I_3 - \frac{\sin\theta_k}{|\mathrm{\omega}_k|}\,\Omega_k
+ \frac{1-\cos\theta_k}{|\mathrm{\omega}_k|^2}\,\Omega_k^2 .
\label{eq:Phi22}
\end{equation}
Using this result, a first-order discretization of the continuous-time 
state and input matrices in~\eqref{eq:state_command_matrix_G} yields
\begin{equation}
A^G_{d,k} \approx
\begin{bmatrix}
1 & T & \tfrac{T^{2}}{2}\,\mathrm{a}_k^\top \\
0 & 1 & T\,\mathrm{a}_k^\top \\
0 & 0 & \phi_{22,k}
\end{bmatrix},\qquad
B^G_{d,k} =
\begin{bmatrix}
\tfrac{T^{2}}{2} \\
T \\
\mathbf 0_{3\times 1}
\end{bmatrix}.
\label{eq:AdBd_G}
\end{equation}
The resulting discrete-time observer, summarized in 
Algorithm~\ref{algo:Discrete_proposed_obs_G}, follows the standard 
correction--prediction structure with the above state and input matrices. The attitude observer~\eqref{eq:attitude_observer_G} is discretized at the 
IMU frequency using exponential-Euler integration on $\mathrm{SO}(3)$~\cite{mahony2008nonlinear} (see line~25 of Algorithm~\ref{algo:Discrete_proposed_obs_G}). 

\begin{algorithm}[!t]
\caption{Discrete-Time Implementation of the proposed cascade architecture}
\label{algo:Discrete_proposed_obs_G}
\begin{algorithmic}[1]
\Statex \textbf{Input:} $\hat x^G_{0|0}, \hat h^G_{0},\hat v^G_{h,0},\hat z_{0},\,P^G_{0|0},\hat R^G_0$,$g_k$,$\mathrm{a}_k$,$\mathrm{\omega}_k$, $\mathrm{m}_{\mathcal B,k}$, $y_{b,k}$, $T$, $Q^G_k$, $S^G_k$; gains $k_z,k_m$.
\Statex \textbf{Output:} $\hat x^G_k,\,\hat R^G_{k},\,$ \(\forall k \in \mathbb{N}\)
\For{each time \(k \geq 0\)}
\If{IMU data \(\mathrm{a}_k\), \(\omega_k\) is available}
    \State $\Omega_k \gets \mathrm{\omega}_k^\times$,\quad $\theta_k \gets |\omega_k|\,T$, \quad $S^G_{d,k}\gets S^G_k\,T$;
    \State {$\phi_{22,k} \gets $~\eqref{eq:Phi22};}
    \State $A^G_{d,k}, B^G_{d,k} \gets$~\eqref{eq:AdBd_G};
    \State $\hat x^G_{k+1|k} \gets A^G_{d,k}\hat x^G_{k|k} + B^G_{d,k} g_k$;
    \State $P^G_{k+1|k} \gets A^G_{d,k}P^G_{k|k}(A_{d,k}^G)^\top + S^G_{d,k}$
\EndIf
\If {barometer data is available}
    \State $C^G_k \gets$\eqref{eq:output_matrix_G}; 
    \State $K^G_{k+1} \gets P^G_{k+1|k} (C^G_k)^\top \big(C^G_kP^G_{k+1|k} (C^G_k)^\top + Q^G_{k}\big)^{-1}$
    \State \(\hat x^G_{k+1|k+1} \gets \hat x^G_{k+1|k}+ K_{k+1}\left(y_{b,k+1}-C^G_k \hat x^G_{k+1|k}\right)\)
    \State \(P_{k+1|k+1} \gets (I_5-K^G_{k+1}C_k)P^G_{k+1|k}\)
\Else
    \State $\hat x^G_{k+1|k+1}\gets \hat x^G_{k+1|k}$; $P^G_{k+1|k+1}\gets P^G_{k+1|k}$.
\EndIf
\State $P^G_{k+1|k+1}\gets \frac{1}{2}(P^G_{k+1|k+1} + (P^G_{k+1|k+1})^{\top})$
\State $\hat h_{k+1}^G \gets \hat x^G_{k+1|k+1}(1)$; $\hat v_{h,k+1}^G \gets \hat x^G_{k+1|k+1}(2)$; 
\State $\hat z_{k+1} \gets \hat x^G_{k+1|k+1}(3:5)$
\If{Magnetometer data is available or k= 0}
\State $\bar{\mathrm{m}}_{\mathcal I} \gets \Pi_{\mathbf e_3}\mathrm{m}_{\mathcal I}$;\quad $\bar{\mathrm{m}}_{\mathcal B,k} \gets \bar\Pi_{\hat z_k}\mathrm{m}_{\mathcal B,k}$
\State $\sigma_{R_{m,k}} \gets k_m(\bar{\mathrm{m}}_{\mathcal I,k} \times \hat R_k \bar{\mathrm{m}}_{\mathcal B,k})$
\EndIf
\State $\sigma_{R,k} \gets k_z(\mathbf e_3 \times \hat R_k \hat z_k) + \sigma_{R_{m,k}}$
\State $\hat R_{k+1} \gets \hat R_k \exp\!\big((\mathrm{\omega}_k - \hat R_k^\top \sigma_{R,k})^\times T\big)$
\EndFor
\end{algorithmic}
\end{algorithm}

\subsection{One-Stage Observer Architecture}
This subsection presents the proposed unified observer architecture for
barometer-aided attitude estimation. The observer evolves on
$\mathrm{SO}(3)\times\mathbb{R}^2$ and directly fuses IMU, barometric,
and magnetometer measurements within a unified nonlinear framework.
Its innovation terms are generated by a local Riccati
equation~\eqref{eq:creL}, derived from the linearized
altitude--vertical velocity--attitude error dynamics, and injected into
the observer dynamics~\eqref{eq:observer_dyn}. The resulting observer
provides estimates of both the full attitude and the vertical motion
states. The overall architecture is depicted in
Fig.~\ref{fig:observer_block_diagram_L}.

\begin{figure}[ht]
\centering
\begin{tikzpicture}[
    font=\small,
    >=Latex,
    node distance=1.2cm,
    sensor/.style={
        draw, thick, rounded corners=2mm,
        fill=gray!12,
        minimum width=2.0cm,
        minimum height=0.7cm,
        align=center
    },
    block/.style={
        draw, thick, rounded corners=2mm,
        minimum width=2.2cm,
        align=center
    },
    riccati/.style={block, fill=blue!12,minimum height=2.2cm,},
    attitude/.style={block, fill=green!12,minimum height=2.0cm},
    arrow/.style={->, thick},
    lab/.style={font=\scriptsize, fill=white, inner sep=1.5pt}
]

\node[riccati]  (ric) at (0,  0.8)
{   \textbf{Riccati Equation}~\eqref{eq:creL}\\\\[-1mm]
    \textbf{Innovation Terms}~\eqref{eq:innovation_inputs}
};

\node[attitude] (att) at (0, -2.8)
{\textbf{Nonlinear Observer}\\[-1mm]\scriptsize \eqref{eq:observer_dyn}};

\node[sensor] (baro) at (-5.2,  2.) {Barometer};
\node[sensor] (mag)  at (-5.2, 0.8) {Magnetometer};
\node[sensor] (ref)  at (-5.2, -0.4) {Reference field\\[-1mm]\scriptsize $m_{\mathcal I}$};
\node[sensor] (imu)  at (-5.2,  -1.8) {IMU\\[-1mm]\scriptsize acc. + gyro};

\coordinate (ric_h) at ($(ric.west)+(0,  0.75)$);
\coordinate (ric_a) at ($(ric.west)+(0, 0.38)$);
\coordinate (ric_mB) at ($(ric.west)+(0, 0)$);
\coordinate (ric_mI) at ($(ric.west)+(0,-0.38)$);
\coordinate (ric_est) at ($(ric.east)+(0, 0.0)$);

\coordinate (att_d1)  at ($(att.west)+(0,0.5)$);
\coordinate (att_d2)  at ($(att.west)+(0,0.25)$);
\coordinate (att_d3)  at ($(att.west)+(0,0)$);
\coordinate (att_a)  at ($(att.west)+(0,-0.25)$);
\coordinate (att_w)  at ($(att.west)+(0,-0.5)$);

\coordinate (atthhat) at ($(att.east)+(0,0.25)$);
\coordinate (att_dhhat) at ($(att.east)+(0,-0.25)$);
\coordinate (attRhat) at ($(att.east)+(0,0)$);

\coordinate (imu_split_a) at ($(imu.east)+(0,  0.18)$);
\coordinate (imu_split_w) at ($(imu.east)+(0,  -0.18)$);
\coordinate (a_fork) at ($(imu_split_a)+(0.6,  0)$);

\draw[arrow] (baro.east) -- ++(1.1,0) |- 
    node[lab, near start, above] {$h$} (ric_h);
    
\draw[arrow] (mag.east) -- ++(1.1,0.0) |- 
    node[lab, near start, above] {$m_{\mathcal B}$} (ric_mB);

\draw[arrow] (ref.east) -- ++(1.1,0) |- 
    node[lab, near start, above] {$m_{\mathcal I}$} (ric_mI);
    
\draw[-,thick](imu_split_a)--
    node[lab,above]{$a$}(a_fork);
\draw[arrow,thick] (a_fork) |- (ric_a);
\draw[arrow,thick] (a_fork) |- (att_a);
\fill(a_fork) circle(1.5pt);

\draw[arrow] (imu_split_w) -- ++(0.25,0.0) |- 
    node[lab, pos=0.4, above] {$\omega$} (att_w);

\coordinate (zdown) at ($(ric.south)+(0,-0.5)$);

\draw[arrow] (ric.south) -- (zdown) 
    -- ++(-2.8,0) |-
    node[lab, pos = 0.82, above] {$\delta_1$} (att_d1);

\draw[arrow] (ric.south) -- (zdown) 
    -- ++(-2.8,0) |-
    node[lab, pos = 0.82, above] {$\delta_2$} (att_d2);

\draw[arrow] (ric.south) -- (zdown) 
    -- ++(-2.8,0) |-
    node[lab, pos = 0.82, above] {$\delta_3$} (att_d3);

\coordinate (rail_top) at ($(atthhat)+(0,2.5)$);
\coordinate (rail_hhat) at ($(rail_top)+(1,0)$);
\coordinate (rail_dhhat) at ($(att_dhhat)+(1,0)$);
\coordinate (rail_Rhat) at ($(attRhat)+(1,0)$);

\draw[-,thick](atthhat) -- (rail_hhat)
    node[lab,above, pos =0.4]{$\hat{h}^L,\hat{v}_h^L,\hat{R}^L$};

\coordinate(rail_ric) at ($(ric_est -| rail_hhat)$);
\draw[-,thick](rail_hhat) --(rail_ric);
\draw[arrow](rail_ric) --(ric_est);
\end{tikzpicture}

\caption{Overall structure of the proposed one-stage observer architecture.}
\label{fig:observer_block_diagram_L}
\end{figure}
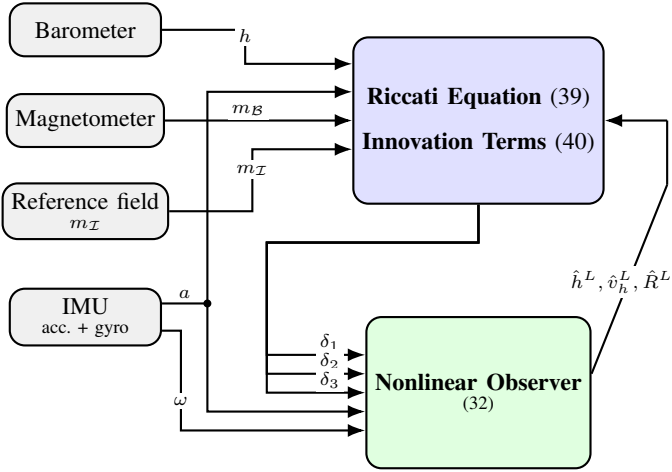
\subsubsection{Full Attitude Estimation}

Recall the altitude variable \(h:=\mathrm{e}_3^\top p\). From
\eqref{eq:position_dyn}--\eqref{eq:velocity_dyn}, the vertical dynamics
satisfy
\begin{align}
\dot h &= v_h, \label{eq:h_dot}\\
\dot v_h &= g+\mathrm{e}_3^\top R\mathrm{a}.
\label{eq:dh_dot}
\end{align}
Combining these equations with the attitude kinematics gives
\begin{equation}
\begin{cases}
\dot h = v_h,\\
\dot v_h = g+\mathrm{e}_3^\top R\mathrm{a},\\
\dot R = R\omega^\times.
\end{cases}
\label{eq:syst_dyn}
\end{equation}
We propose the nonlinear observer on
\(\mathrm{SO}(3)\times\mathbb{R}^2\)
\begin{equation}
\begin{cases}
\dot{\hat h}^L = \hat v_h^L+\delta_1,\\
\dot{\hat v}_h^L = g+\mathrm{e}_3^\top \hat R^L\mathrm{a}+\delta_2,\\
\dot{\hat R}^L = \hat R^L\omega^\times+\delta_3^\times\hat R^L,
\qquad \hat R^L(0)\in\mathrm{SO}(3),
\end{cases}
\label{eq:observer_dyn}
\end{equation}
where \(\hat h^L\), \(\hat v_h^L\), and \(\hat R^L\) denote the estimates
of \(h\), \(v_h\), and \(R\), respectively. The innovation terms
\(\delta_1\in\mathbb{R}\), \(\delta_2\in\mathbb{R}\), and
\(\delta_3\in\mathbb{R}^3\) are designed from the local linearized error
model derived below.

Let define the errors as:
\[
\tilde h^L:=h-\hat h^L,\quad
\tilde v_h^L:=v_h-\hat v_h^L,\quad
\tilde R^L:=R(\hat R^L)^\top.
\]
Then, from \eqref{eq:syst_dyn} and \eqref{eq:observer_dyn},
\begin{subequations}
\label{eq:error_dyn}
\begin{align}
\dot{\tilde h}^L &= \tilde v_h^L-\delta_1,\\
\dot{\tilde v}_h^L
&= \mathrm{e}_3^\top(\tilde R^L-I_3)\hat R^L\mathrm{a}-\delta_2,\\
\dot{\tilde R}^L &= -\tilde R^L\delta_3^\times .
\end{align}
\end{subequations}
For local analysis, we parameterize the attitude error by
\[
\tilde R^L
=
I_3+\tilde\lambda^\times+\mathcal O(|\tilde\lambda|^2),
\]
where \(\tilde\lambda\in\mathbb{R}^3\) is the first-order attitude-error
coordinate, obtained for instance from the quaternion representation of
\(\tilde R^L\). Substitution into \eqref{eq:error_dyn} yields
\begin{subequations}
\label{eq:error_dyn_first_order}
\begin{align}
\dot{\tilde h}^L &= \tilde v_h^L-\delta_1,\\
\dot{\tilde v}_h^L
&=
-\mathrm{e}_3^\top(\hat R^L\mathrm{a})^\times\tilde\lambda
-\delta_2
+\mathcal O(|\tilde\lambda|^2),\\
\dot{\tilde\lambda}
&=
-\delta_3+\mathcal O(|\tilde\lambda||\delta_3|).
\end{align}
\end{subequations}
The estimated outputs are defined by
\[
\hat y_b^L:=\hat h^L,
\qquad
\hat y_m^L:=\hat R^L\mathrm{m}_{\mathcal B}.
\]
Hence, using \eqref{eq:baro_measurement}--\eqref{eq:mag_measurement},
the local output errors satisfy
\begin{subequations}
\label{eq:output_error}
\begin{align}
\tilde y_b^L &:= y_b-\hat y_b^L = \tilde h^L,\\
\tilde y_m^L &:= \mathrm{m}_{\mathcal I}-\hat y_m^L
=
-(\mathrm{m}_{\mathcal I})^\times\tilde\lambda
+\mathcal O(|\tilde\lambda|^2).
\end{align}
\end{subequations}

Let
\[
x^L:=
\begin{bmatrix}
\tilde h^L & \tilde v_h^L & \tilde\lambda^\top
\end{bmatrix}^\top,
\qquad
u:=
\begin{bmatrix}
-\delta_1 & -\delta_2 & -\delta_3^\top
\end{bmatrix}^\top,
\]
and
\[
y^L:=
\begin{bmatrix}
\tilde y_b^L & (\tilde y_m^L)^\top
\end{bmatrix}^\top .
\]
Then the local error dynamics can be written as
\begin{subequations}
\label{eq:ltv_L}
\begin{align}
\dot x^L
&=
A^L(t)x^L+u
+\mathcal O(|\tilde\lambda||u|+|\tilde\lambda|^2),\\
y^L
&=
C^Lx^L+\mathcal O(|\tilde\lambda|^2),
\end{align}
\end{subequations}
where
\begin{equation}
A^L(t)=
\begin{bmatrix}
0 & 1 & 0_{1\times3}\\
0 & 0 & -\mathrm{e}_3^\top(\hat R^L\mathrm{a})^\times\\
0_{3\times1} & 0_{3\times1} & 0_{3\times3}
\end{bmatrix},
\label{eq:state_matrix_L}
\end{equation}
and
\begin{equation}
C^L=
\begin{bmatrix}
1 & 0 & 0_{1\times3}\\
0_{3\times1} & 0_{3\times1} & -(\mathrm{m}_{\mathcal I})^\times
\end{bmatrix}.
\label{eq:est_output_matrix}
\end{equation}
The innovation is chosen as
\[
u(t)=-K^L(t)y^L(t),
\qquad
K^L(t)=P^L(t)(C^L)^\top(Q^L)^{-1},
\]
where \(P^L(t)\) is the solution of the continuous-time Riccati equation
\begin{align}
\dot P^L
&=
A^L(t)P^L+P^L A^L(t)^\top \notag\\
&\quad
-P^L(C^L)^\top(Q^L)^{-1}C^L P^L+S^L,
\qquad P^L(0)>0.
\label{eq:creL}
\end{align}
The matrices \(Q^L>0\) and \(S^L>0\) are symmetric positive definite
weighting matrices associated with measurement and process
uncertainties, respectively.

Partitioning the gain as
\[
K^L(t)=
\begin{bmatrix}
K_{\delta_1}(t)^\top &
K_{\delta_2}(t)^\top &
K_{\delta_3}(t)^\top
\end{bmatrix}^\top,
\]
with
\[
K_{\delta_1}\in\mathbb{R}^{1\times4},\qquad
K_{\delta_2}\in\mathbb{R}^{1\times4},\qquad
K_{\delta_3}\in\mathbb{R}^{3\times4},
\]
the innovation terms are given by
\begin{subequations}
\label{eq:innovation_inputs}
\begin{align}
\delta_1(t) &= K_{\delta_1}(t)y^L(t),\\
\delta_2(t) &= K_{\delta_2}(t)y^L(t),\\
\delta_3(t) &= K_{\delta_3}(t)y^L(t).
\end{align}
\end{subequations}
The stability of the observer is governed by the well-posedness of the
CRE~\eqref{eq:creL}. Let
\[
A^{L\star}(t):=A^L(R(t))
\]
denote the state matrix evaluated along the true trajectory. If the pair
\((A^{L\star}(t),C^L)\) is uniformly observable, then
\eqref{eq:creL} admits a unique bounded positive-definite solution
\(P^L(t)\) for all \(t\ge0\). Consequently, the origin of the
linearized error system is exponentially stable, and the nonlinear
observer~\eqref{eq:observer_dyn} is locally exponentially convergent
around the true trajectory~\cite{hamel2017riccati}.

\subsubsection{Observability Analysis}\label{sec:obs_stab_analysis}
We now establish sufficient conditions for the uniform observability
(UO) of the pair \((A^{L\star}(t),C^L)\), where
\(A^{L\star}(t):=A^L(R(t))\) denotes the system matrix evaluated along
the true trajectory. Unlike the two-stage observer, the proposed
unified architecture only requires persistent excitation of the
horizontal component of the inertial acceleration. This weaker
condition avoids explicit computation of the observability Gramian and
can be verified directly from the vehicle motion.
\begin{lemma} \label{Lemma2}
Assume that the body-frame linear acceleration $\mathrm{a}(t)$ and the angular velocity $\mathrm{\omega}(t)$ are continuous and uniformly bounded. Moreover, assume that vectors $\mathrm{m}_\mathcal{I}$ and $\mathrm{e}_3$ are non-collinear. Let $J=[\mathrm{e}_1\ \mathrm{e}_2]\in\mathbb{R}^{3\times2}$ and define $a_\perp(t)=\mathrm{e}_3^\top(R(t)\mathrm{a}(t))^\times J$.
Assume there exist $\bar\delta>0$ and $\bar\mu>0$ such that, for all $t\ge0$,
\begin{equation}
\frac{1}{\bar\delta}\int_t^{t+\bar\delta}a_\perp(s)^\top a_\perp(s)\,ds
\ \ge\ \bar\mu I_2.
\label{eq:PE_condition}
\end{equation}
Then the pair $(A^{L\star}(t),C)$ is \emph{uniformly observable}.
\end{lemma}
\begin{IEEEproof}
See Appendix~\ref{appendix_C}.
\end{IEEEproof}
The excitation condition in~\eqref{eq:PE_condition} admits the
geometric interpretation illustrated in
Fig.~\ref{fig:a_perp_geometry}. The inertial acceleration
\[
a_{\mathcal I}
=
\begin{bmatrix}
a_{\mathcal I,1}&
a_{\mathcal I,2}&
a_{\mathcal I,3}
\end{bmatrix}^{\!\top}
\]
can be decomposed into its horizontal and vertical components as
\[
a_{\mathcal I}
=
\Pi_{\mathrm e_3}a_{\mathcal I}
+
(\mathrm e_3^\top a_{\mathcal I})\mathrm e_3
=
a_{\mathcal I,1,2}
+
a_{\mathcal I,3},
\]
where
\[
a_{\mathcal I,1,2}
=
\Pi_{\mathrm e_3}a_{\mathcal I}
=
\begin{bmatrix}
a_{\mathcal I,1}&
a_{\mathcal I,2}&
0
\end{bmatrix}^{\!\top}
\]
lies in the horizontal plane \(\Pi\) orthogonal to
\(\mathrm e_3\), while
\[
a_{\mathcal I,3}
=
(\mathrm e_3^\top a_{\mathcal I})\mathrm e_3
\]
is aligned with the gravity direction.

The vector \(a_\perp\) introduced in
Lemma~\ref{Lemma2} is given by
\[
a_\perp
=
\mathrm e_3^\top
a_{\mathcal I}^\times
J
=
\begin{bmatrix}
-a_{\mathcal I,2}&
a_{\mathcal I,1}
\end{bmatrix}^{\!\top},
\]
which corresponds to a \(90^\circ\) rotation of
\(a_{\mathcal I,1,2}\) within the horizontal plane.
Consequently, \(a_\perp\) carries exactly the same excitation
information as the horizontal component of the inertial acceleration.
Therefore, the observability condition of
Lemma~\ref{Lemma2} only requires persistent excitation of the
horizontal component of the inertial acceleration and imposes no
excitation requirement along the gravity direction. This condition is
strictly weaker than the AGAS condition
\eqref{eq:PE_Condition_G}, which requires persistent excitation of the
full inertial acceleration vector. The proposed unified observer can
therefore guarantee observability under significantly less restrictive
vehicle motions.

\begin{figure}[t]
\centering
\begin{tikzpicture}[scale=1.0, >=stealth]

\definecolor{myblue}{RGB}{0,70,180}
\definecolor{mygreen}{RGB}{0,130,40}
\definecolor{myred}{RGB}{200,40,30}
\definecolor{mypurple}{RGB}{120,40,60}

\coordinate(O) at (0,0);
\coordinate(e1) at (1.,1.);
\coordinate(e2) at (2.0,0);
\coordinate(e3) at (0,-1.5);

\coordinate(P1) at (-0.8,-0.35);
\coordinate(P2) at (4.2,-0.35);
\coordinate(P3) at (5.4,1.05);
\coordinate(P4) at (0.4,1.05);

\coordinate(Ah) at (2.5,0.75); 
\coordinate(A) at (2.55,2); 

\fill[blue!8] (P1)--(P2)--(P3)--(P4)--cycle;
\draw[blue!45, dashed](P1) --(P2)--(P3)--(P4)--cycle;
\node[blue!70!black] at (4.95,0.85){$\Pi$};
\draw[blue!70!black,thick]
(4.6,1) arc[start angle=170, end angle = 250, radius = 0.35];
\draw[->,thick](O) --(e1) node[above]{$\mathrm{e}_1$};
\draw[->,thick](O) --(e2) node[right]{$\mathrm{e}_2$};
\draw[->,thick](O) --(e3) node[right]{$\mathrm{e}_3$};
\node[left] at (O) {$\mathcal{I}$};

\draw[->,very thick, myblue](O)--(A)
node[above left]{$a_{\mathcal I}$};

\draw[->,very thick, mygreen](O)--(Ah)
node[midway, below]{$a_{\mathcal I,1,2}$};

\draw[->, thick, myred, dashed] (Ah) -- (A)
node[midway, right]{$a_{\mathcal I,3}$};

\draw[dashed] (A) -- (Ah);
\fill (A) circle (1.5pt);
\fill (Ah) circle (1.5pt);

\draw[->, very thick, mypurple, dashed] (Ah) -- ++(0.9,-0.98)
node[midway,right]{$a_\perp$};

\end{tikzpicture}    
\caption{Geometric interpretation of the LES Observer excitation in the inertial frame \( \mathcal{I} \). The inertial acceleration $a_{\mathcal I} \in \mathbb{R}^3$ is decomposed into a horizontal component $a_{\mathcal I,1,2} \in \mathbb{R}^3$ lying in the plane $\Pi=\mathrm{span}\{\mathrm{e}_1,\mathrm{e}_2\}$, and a vertical component $a_{\mathcal I,3} \in \mathbb{R}^3$ aligned with $\mathrm{e}_3$. The vector $a_\perp \in \mathbb{R}^2$ is the planar coordinate representation of $a_{\mathcal I,1,2}$ and carries the excitation information used in the observability condition of Lemma~\ref{Lemma2} .}
\label{fig:a_perp_geometry}
\end{figure}
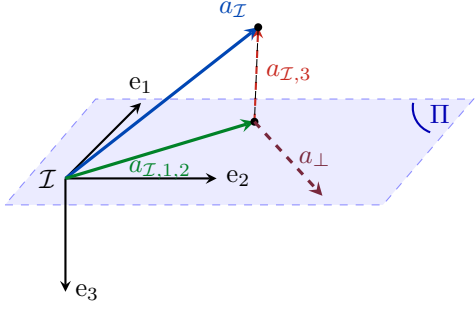

The following theorem is a direct consequence of \cite[Theorem 3.1]{hamel2017riccati}; consequently, its proof is omitted.

\begin{theorem}\label{Theorem2}
Consider the system~\eqref{eq:position_dyn}-\eqref{eq:mag_measurement}, together with the nonlinear observer~\eqref{eq:observer_dyn} and the associated closed-loop error dynamics~\eqref{eq:error_dyn_first_order}, where the innovation terms are defined in~\eqref{eq:innovation_inputs} and $P^L(t)$ denotes the symmetric positive-definite solution of the Riccati equation~\eqref{eq:creL}. If the pair \((A^{L\star}(t), C^L)\) is uniformly observable, then the origin of the error system~\eqref{eq:error_dyn_first_order} is locally exponentially stable.
\end{theorem}

Theorem~\ref{Theorem2} establishes that the Riccati-based gain $K^L(t)$, synthesized from the linearized time-varying error model, is sufficient to guarantee local exponential stability of the origin of the nonlinear error dynamics~\eqref{eq:error_dyn_first_order}, provided the uniform observability condition of Lemma~\ref{Lemma2} holds. Consequently, the observer gain can be synthesized entirely from the linearized error dynamics while retaining local exponential stability of the original nonlinear estimation error system.

\subsubsection{Discrete-Time Implementation}\label{sec:implementation_L}

Similarly to the AGAS observer implementation in Section \ref{sec:implementation_G},  the LES observer is implemented in discrete time with sampling period~$T$ corresponding to the IMU rate. Over each interval $[t_k,t_{k+1}]$, the measured angular velocity $\mathrm{\omega}_k$ and the specific force $\mathrm{a}_k$ are assumed constant, consistent with a zero-order hold approximation.
In addition, the discrete process noise covariance is approximated as $S^L_{d,k}\approx S^L_k T$, and the estimated attitude is treated as piecewise constant, i.e., \(\hat R^L(t) \approx \hat R^L_k\) for \( t \in [t_k,t_{k+1}]\).
Under these assumptions, a first-order discretization of the continuous-time state matrix $A^L(t)$ yields
\begin{equation}
A^L_{d,k} \approx
\begin{bmatrix}
1 & T  & 0_{1\times 3} \\
0 &  1 & -T\mathrm{e}_{3}^{\top}(\hat R^L_k \mathrm{a}_k)^\times \\
0_{3\times 1}& 0_{3\times 1} &I_3
\end{bmatrix}. \label{eq:Ad}
\end{equation}
The output matrix $C^L$ is obtained at the measurement timestamp as follows
\begin{equation}
C^L_{k} = \begin{bmatrix}C_{b,k}^\top , C_{m,k}^\top\end{bmatrix}^\top \label{eq:Ck}
\end{equation}
where
\begin{equation}
 C_{b,k}
= \begin{bmatrix}
1 & 0 & 0_{1 \times 3}
\end{bmatrix}, \label{eq:cbk}
\end{equation}
\begin{equation}
 C_{m,k}
= \begin{bmatrix}
0_{3 \times 1} & 0_{3 \times 1} & -(\mathrm{m}_{\mathcal{I}})^{\times}
\end{bmatrix}. \label{eq:cmk}
\end{equation}
The resulting discrete-time implementation of the proposed LES observer, summarized in Algorithm~\ref{algo:Discrete_proposed_obs_L}, 
follows a prediction--correction structure based on the matrices~in \eqref{eq:Ad}-\eqref{eq:Ck}. The nonlinear dynamics in~\eqref{eq:observer_dyn} are discretized at the  IMU frequency using exponential-Euler and forward Euler step integrations on $\mathrm{SO}(3)\times\mathbb{R}^2$ \cite{mahony2008nonlinear}(see lines~25--27 of Algorithm~\ref{algo:Discrete_proposed_obs_L}).
\begin{algorithm}[!t]
\caption{Discrete-Time Implementation of the proposed unified observer}
\label{algo:Discrete_proposed_obs_L}
\begin{algorithmic}[1]
\Statex \textbf{Inputs:} $x^L_{0|0}, P^L_{0|0},\hat R^L_0, \hat v^L_{h,0}, \hat h^L_{0}$; $g_k,\mathrm{a}_k,\mathrm{\omega}_k, \mathrm{m}_{\mathcal{B}_k}$; $y_{m,k}$; $y_{b,k}$; $T$; $Q^L_k$; $S^L_k$.
\Statex \textbf{Outputs:} $\,\hat R^L_{k},\,\hat v^L_{h,k}, \, \hat h^L_k, \,$ \(\forall k \in \mathbb{N}\)
\For{each time \(k \geq 0\)}
\If{IMU data \(\mathrm{a}_k\), \(\omega_k\) is available}
    \State {$A^L_{d,k} \gets $~\eqref{eq:Ad};} \quad $S^L_{d,k}\gets S^L_k\,T$;
    \State{\(P^L_{k+1|k} \gets A^L_{d,k} P^L_{k|k} (A^L_{d,k})^\top + S^L_{d,k}.\)}
\EndIf
\If{$y_{m,k}$ and $y_{b,k}$ are all available}
\State $C_{m,k}~\gets$ \eqref{eq:cmk}; $C_{b,k}~\gets$ \eqref{eq:cbk}; $C^L_k \gets$~\eqref{eq:Ck};
\State $y^L_k\gets \begin{bmatrix} y_{b,k} - \hat h^L_{k}\\ y_{m,k} - \hat R^L_k \mathrm{m}_{\mathcal{B}_k}
\end{bmatrix};$ 
\State $Q^L_k \gets \mathrm{blkdiag}(Q_{b,k},Q_{m,k})$.
\ElsIf{Barometer measurement is available}
\State $C_{b,k} \gets$~\eqref{eq:cbk}; $C^L_k \gets C_{b,k}$;
\State \( y^L_{k} \gets y_{b,k} - \hat h^L_{k}\); $Q^L_{k} \gets Q_{b,k}$.
\ElsIf{Magnetometer measurement is available}
\State $C_{m,k} \gets$~\eqref{eq:cmk}; $C^L_k \gets C_{m,k}$;
\State \( y^L_{k} \gets  y_{m,k} - \hat R^L_k \mathrm{m}_{\mathcal{B}_k}\); $Q^L_{k} \gets Q_{m,k}$.
\EndIf
\If{ $y_{b,k}$ or $y_{m,k}$ or all are available}
\State $K^L_k = P^L_{k+1|k}\,(C^L)_k^\top \big(C^L_k P^L_{k+1|k} (C^L)_k^\top + Q^L_k\big)^{-1}$;
\State $P^L_{k+1|k+1}=\left(I_5 - K^L_k C^L_k \right)P^L_{k+1|k}$;
\State $u_k \gets-K^L_k y^L_k$; $\delta_{1,k}\ \gets -u_k(1)$; $\delta_{2,k} \gets -u_k(2)$ $\delta_{3,k} \gets -u_k(3{:}5)$.
\Else
    \State $u_k\gets 0$;$P^L_{k+1|k+1}\gets P^L_{k+1|k}$;$\delta_{1,k},\delta_{2,k},\delta_{3,k} \gets 0$. 
\EndIf
\State $P^L_{k+1|k+1}\gets \frac{1}{2}(P^L_{k+1|k+1} + (P^L)_{k+1|k+1}^{\top})$
\State $\hat R^L_{k+1} \gets \hat R^L_k \exp\!\big((\mathrm{\omega}_k + (\hat R^L)_k^\top \delta_{3,k})^\times T\big)$;
\State $\hat v^L_{h, k+1} \gets \hat v^L_{h,k}+T\left( g + \mathrm{e}_3^\top \hat R^L_k\mathrm{a}_k + \delta_{2,k} \right )$;
\State $\hat h^L_{k+1} \gets \hat h^L_k + T\left(v^L_{h,k}+ \delta_{1,k} \right)$
\EndFor
\end{algorithmic}
\end{algorithm}

\section{Simulation and Experimental Results}\label{sec:sim_exp}

This section evaluates the proposed observers using both simulated and experimental data. The discrete-time implementations introduced in Sections~\ref{sec:implementation_G} and \ref{sec:implementation_L} are applied to two distinct datasets: one generated through numerical simulations and another obtained from measurements collected during real flight experiments. To assess the practical implications of the LES and AGAS properties, the two observers are compared in terms of convergence behavior, steady-state accuracy, and robustness to initialization. The corresponding uniform observability conditions are evaluated by computing the condition number of the associated persistent excitation matrices:
\begin{equation}
M^G(t_k,t_{k+1}) := \frac{1}{\bar{\delta}}\int_t^{t+\bar{\delta}}a_{\mathcal{I}}(s)a_{\mathcal{I}}(s)^\top\,ds, \label{eq:cond_numb_AGAS}
\end{equation}

\begin{equation}
M^L(t_k,t_{k+1}) := \frac{1}{\bar\delta}\int_t^{t+\bar\delta}a_\perp(s)^\top a_\perp(s)\,ds, \label{eq:cond_numb_LES} 
\end{equation}
over the interval $[t_k,t_{k+1}]$ with $t_k =k\delta$, $k$ a positive integer, $\delta = 2 (s)$.

\subsection{Simulation Results}
The simulation considers a rigid-body vehicle equipped with an IMU and a barometer evolving in three-dimensional space. The ground-truth angular velocity $(\mathrm{rad/s})$ and inertial altitude $(\mathrm{m})$ are defined as
\[
\omega(t)=
\begin{bmatrix}
0.4\sin(0.5t)\\
0.5\sin(0.3t+\pi/4)\\
0.3\sin(0.7t+\pi/3)
\end{bmatrix},
\qquad
h(t)=-\frac{5\sqrt{3}}{4}\sin(2t),
\]
while the body-frame specific force $(\mathrm{m/s^2})$ is generated according to
\[
a(t)=R^\top
\left(
\begin{bmatrix}
-\cos(t)\\
-\sin(2t)\\
5\sqrt{3}\sin(2t)
\end{bmatrix}
-g
\right),
\]
with the initial condition $R(0)=I_3$.

Two scenarios are considered to investigate the complementary convergence properties of the two observers. Scenario~1 corresponds to small initial estimation errors, whereas Scenario~2 considers large initialization errors. For each scenario, $50$ Monte Carlo simulations are performed, with the initial estimates randomly drawn from Gaussian distributions whose parameters are summarized in Table~\ref{tab:init_conditions_sim}. The initial attitude estimates $\hat R^L(0)$ and $\hat R^G(0)$ are parameterized by Euler angles (yaw, pitch, roll), each perturbed according to the corresponding standard deviation reported in the table. The initial Riccati matrices $P^L(0)$ and $P^G(0)$ are chosen as specified in Table~\ref{tab:initial_P}.

The inertial magnetic field is set to
$
m_{\mathcal I}=
\begin{bmatrix}
1/\sqrt{2} & 0 & 1/\sqrt{2}
\end{bmatrix}^{\top}.
$
The IMU measurements are sampled at $250 \ \mathrm{Hz}$ and corrupted by zero-mean Gaussian noise with standard deviations of $0.1\ \mathrm{m/s^2}$ for the accelerometer and $0.05\ \mathrm{rad/s}$ for the gyroscope on each axis. The body-frame magnetometer measurements, available at $50\ \mathrm{Hz}$, are corrupted by zero-mean Gaussian noise with standard deviation
$
\sigma_m=
\begin{bmatrix}
0.02 & 0.02 & 0.02
\end{bmatrix}^{\top},
$
yielding the covariance matrix
$
Q_m=\operatorname{diag}
(\sigma_{m,x}^2,\sigma_{m,y}^2,\sigma_{m,z}^2).$
The barometer is sampled at $5\ \mathrm{Hz}$ and corrupted by zero-mean Gaussian noise with covariance
$
Q_b=\sigma_b^2=2.5\times10^{-3}.
$

The observer parameters are summarized in Table~\ref{tab:observer_gains}. For the AGAS observer, the tilt and magnetic corrections are governed by the gains $k_z$ and $k_m$, respectively, while the process covariance is selected as
\[
S^G=\operatorname{diag}(10^{-1},10^{-1},10^{-2},10^{-2},10^{-2}).
\]

The simulation results shown in Fig.~\ref{fig:Sim_Euler_Angles_Att_Error_Scenario_1}--\ref{fig:Sim_Euler_Angles_Att_Error_Scenario_2} demonstrate that both observers successfully recover the vehicle attitude in both initialization scenarios.

For Scenario~1, corresponding to small initial estimation errors, the LES observer converges significantly faster than the AGAS observer. The attitude estimation error reaches negligible values after approximately $9\mathrm{s}$, compared with about $14\mathrm{s}$ for the AGAS observer, while exhibiting a smaller transient envelope. This behavior is consistent with the local exponential convergence predicted by Theorem~\ref{Theorem2}. The Euler-angle estimates remain smooth throughout the transient and rapidly converge to the true attitude.

For Scenario~2, both observers recover from large initial attitude errors and converge to the true attitude. The LES observer again reaches the steady state more rapidly, whereas the AGAS observer exhibits a smoother transient with reduced overshoot. A temporary discontinuity in the AGAS yaw estimate over the interval $t\in[6.6,\,8.3]\mathrm{s}$ results from the magnetic correction, while the LES observer displays larger transient oscillations in the yaw and roll estimates over $t\in[5.3,\,8.3]\mathrm{s}$. These oscillations reflect the reduced accuracy of the local linearization under large initial misalignment and disappear once the estimation error enters the neighborhood in which local exponential convergence is guaranteed. In steady state, both observers achieve comparable estimation accuracy.

These observations are consistent with the persistent excitation analysis shown in Fig.~\ref{fig:Sim_PE_Condition}. The excitation matrix associated with the AGAS observer remains poorly conditioned throughout the trajectory, with $\log(\operatorname{cond}(M^G))\approx31$, whereas the excitation matrix associated with the LES observer remains consistently well conditioned, with $\log(\operatorname{cond}(M^L))\approx0$. Consequently, the Riccati-based LES observer exploits more informative error dynamics, resulting in faster convergence whenever the initial estimation error lies within its region of attraction. In contrast, the AGAS observer retains its almost-global convergence properties despite the weaker excitation, at the expense of slower transient dynamics.
\begin{table}
\centering
\caption{Initial conditions used for the simulation scenarios}
\label{tab:init_conditions_sim}
\renewcommand{\arraystretch}{1.1}
\begin{tabular}{lcc}
\hline
\textbf{Quantity} & \textbf{Scenario 1} & \textbf{Scenario 2}\\
\hline
$\hat{R}^L(0),\hat R^G (0)$ 
& $[15^{\circ},\,9^{\circ},-9^{\circ}]$ 
& $[60^{\circ},\,-30^{\circ},\,45^{\circ}]$\\
Std. of $\hat{R}^L(0),\hat R^G (0)$ 
& $5^{\circ}$ per axis 
&$100^{\circ}$ per axis\\
$\begin{bmatrix}\hat h^L(0)\\ \hat v^L_h(0) \\- \end{bmatrix},$ $\hat x^G$ & 
$\begin{bmatrix}0.5\ \mathrm{m}\\0.5\ \mathrm{m/s} \\ \hat R^G(0)^{\top}\mathrm{e}_3)\\\end{bmatrix}$ 
& $\begin{bmatrix}5\ \mathrm{m}\\5\ \mathrm{m/s}\\ \hat R^G(0)^{\top}\mathrm{e}_3 \end{bmatrix}$\\
Std. of $\begin{bmatrix}\hat h(0)\\ \hat v_h(0) \\ - \end{bmatrix},\hat x^G(0)$ & 
$\begin{bmatrix}1\ \mathrm{m}\\1\ \mathrm{m/s}\\0.05\\0.05\\0.05\end{bmatrix}$ 
& $\begin{bmatrix}8\ \mathrm{m}\\8\ \mathrm{m/s}\\0.5\\0.5\\0.5\end{bmatrix}$\\
\hline
\end{tabular}{}
\end{table}

\begin{table}
\centering
\caption{Initial Riccati Matrices used in simulation}
\label{tab:initial_P}
\renewcommand{\arraystretch}{1}
\begin{tabular}{lcc}
\hline
\textbf{Scenario} & $P^G(0)$ & $P^L(0)$\\
\hline
Scenario~1 
& $\operatorname{diag}(1,1,10^{-2}\cdot(1,1,1))$ 
& $\operatorname{diag}(1,1,10^{-2}\cdot(1,1,1))$ \\
Scenario~2 
& $\operatorname{diag}(25,25,1,1,1)$ 
& $\operatorname{diag}(25,25,1,1,1)$ \\
\hline
\end{tabular}{}
\end{table}

\begin{table}
\centering
\caption{Observer tuning parameters used for the simulation scenarios}
\label{tab:observer_gains}
\renewcommand{\arraystretch}{1.1}
\begin{tabular}{lcc}
\hline
\textbf{Parameters} & \textbf{AGAS Observer} & \textbf{LES Observer}\\
\hline
Tilt gain, $k_z$ 
& $8$ 
& --\\
Magnetic gain, $k_m$ 
& $2.5$ 
& --\\
Process covariance 
& $S^G$ 
& $S^L= S^G$\\
Measurement covariance & 
$Q^G=Q_b$ 
& $Q^L =\operatorname{diag}(Q_b,Q_m)$\\
\hline
\end{tabular}{}
\end{table}

\begin{figure}[!t]
    \centering
    \includegraphics[width=\linewidth]{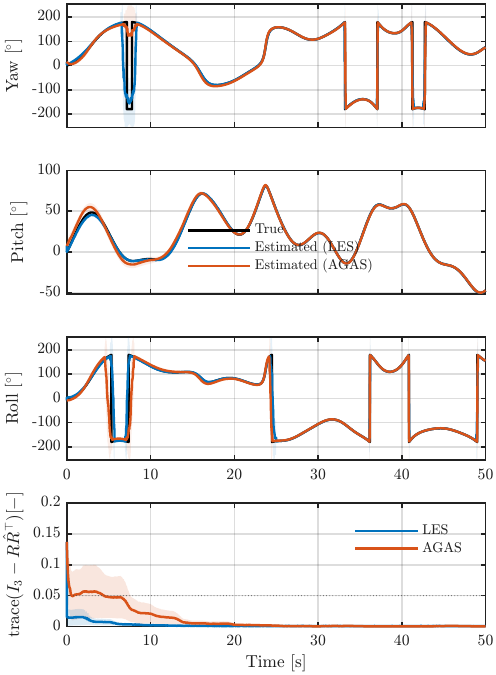}
    \caption{Results using the simulated data. Euler angles and attitude error under small initial values (Scenario~1).}
    \label{fig:Sim_Euler_Angles_Att_Error_Scenario_1}
\end{figure}

\begin{figure}[!t]
    \centering
    \includegraphics[width=\linewidth]{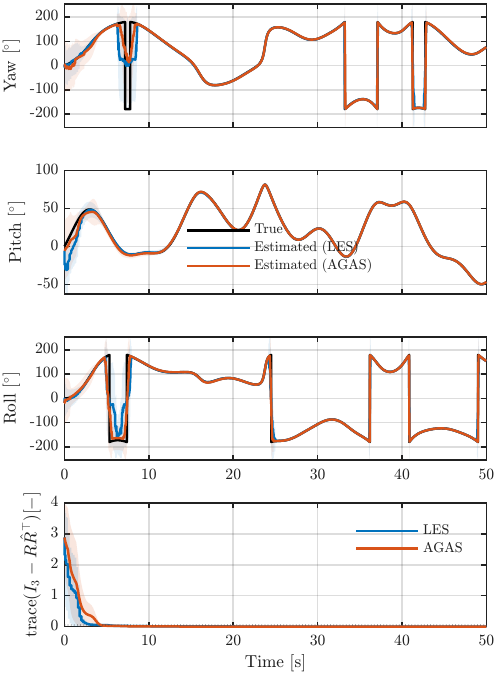}
    \caption{Results using the simulated data. Euler angles and attitude error under large initial values (Scenario~2).}
    \label{fig:Sim_Euler_Angles_Att_Error_Scenario_2}
\end{figure}

\begin{figure}[!t]
    \centering
    \includegraphics[width=\linewidth]{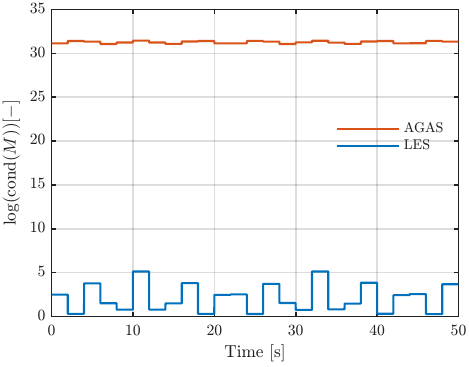}
    \caption{Evolution of cond$(M^L)$ for LES-based observer and cond$(M^G)$ for AGAS-based observer over the time from the simulated data.}
    \label{fig:Sim_PE_Condition}
\end{figure}

\subsection{Experimental Results}
We have shown that, under controlled conditions, both the Barometer--IMU LES- and AGAS-based observers achieve stable and convergent attitude estimation, even under weak excitation conditions. We now evaluate their performance using flight data collected from a MakeFlyEasy Fighter VTOL UAV, shown in Fig.~\ref{fig:veh_baro_imu}, with particular emphasis on robustness to large initial attitude estimation errors.
Flight data were recorded using a \textit{Pixhawk} flight controller running the PX4 autopilot~\cite{meier2015px4}. The onboard sensor suite includes an IMU providing specific force and angular rate measurements \(\mathrm a(t)\) and \(\mathrm{\omega}(t)\), a three-axis magnetometer \( \mathrm{m}_{\mathcal{B}}(t)\in \mathbb{S}^2\) and an altimeter sensor $h(t)$. The reference vertical velocity $v_h(t) :=\dot h(t)$ is obtained by differentiating the measured signal $h(t)$.
The autopilot records flight logs for each onboard sensor and provides estimates of the vehicle attitude \(\bar{R}\) computed internally by an extended Kalman filter \cite{px4_ecl_ekf_tuning}. 

To isolate the contribution of inertial and barometric measurements while avoiding the influence of magnetic disturbances, the attitude estimate $\bar R$ provided by the autopilot is used solely to reconstruct an equivalent inertial magnetic field according to
$
\mathrm{m}_{\mathcal{I}} := \bar R\,\mathrm{m}_{\mathcal{B}}.
$
The reconstructed inertial magnetic field is treated as a known reference for evaluating the full attitude estimation error and is not used by either observer.

The observers are initialized over a selected flight segment, and their performance is assessed using standard convergence and estimation metrics. Let $e_\phi(k)$, $e_\theta(k)$, $e_\psi(k)$ denote the roll, pitch, and yaw estimation errors, respectively, defined as the pointwise differences between estimated and ground-truth Euler angles at time step $k \in \mathbb{N}$. Let $N$ denote the total number of samples in the selected flight segment and $e_{\mathrm{att}}(k) = \operatorname{trace}\!(I_{3} - \bar R(k)\hat R(k)^\top)$ the full attitude estimation error, where $\bar R (k)$ and $\hat R (k)$ denote the ground-truth and estimated rotation matrices at time step $k$, respectively, with $\hat R (k):=\hat R^G(k)$ for AGAS observer and $\hat R (k):=\hat R^L(k)$ for LES observer. The convergence time $t_c$ is defined as the minimum time at which the attitude error $e_{\mathrm{att}}(t)$ enters and permanently remains within a prescribed tolerance band $\epsilon = 0.05$,  i.e.,
\[
t_c  = \min \{t_k : |e_{\mathrm{att}}(t_j)| < \epsilon, \forall j \ge k\}.
\]
The corresponding sample index is denoted by $k_c$.
The maximum attitude estimation error over the full segment is defined as 
\[
e_{\mathrm{att},\max} = \max_{0 \leq k \leq N} \left|e_{\mathrm{att}}(k) \right|,
\] 
while the steady-state attitude error is computed as the mean absolute error over the post-convergence interval$[k_c,N]$: 
\[
e_{\mathrm{att},\mathrm{ss}} = \frac{1}{N-k_c+1}\sum_{k=k_c}^N \left|e_{\mathrm{att}}(k) \right|.
\]
The per-axis estimation accuracy is quantified using the \textit{Root Mean Square Error} (RMSE). The full-segment RMSE values for roll, pitch, and yaw, are defined as
\[
\mathrm{RMSE}_{\beta} = \sqrt{\frac{1}{N}\sum_{k=1}^Ne_{\beta}^2(k)}, 
\]
where $\beta=\{\phi,\theta,\psi\}$. The steady-state RMSE is computed analogously over the interval $[k_c,N].$

Together, these metrics provide a quantitative assessment of the convergence behavior, transient response, and asymptotic estimation accuracy of the proposed observers.
To further quantify the relative performance of the filters, the percentage improvement of the LES observer with respect to the AGAS observer is computed for each performance metric according to 
\[
\mathrm{Improvement}(\%) = 100 \cdot \frac{\mathcal{M}_{\mathrm{AGAS}}-\mathcal{M}_{\mathrm{LES}}}{\mathcal{M}_{\mathrm{AGAS}}},
\]
where $\mathcal{M}_{\mathrm{AGAS}}$ and $\mathcal{M}_{\mathrm{LES}}$ denote the values of the considered metric obtained with AGAS and LES observers, respectively. Since smaller values indicate better estimation performance for all considered metrics, a positive percentage corresponds to an improvement of the LES observer over the AGAS observer.

For both observers, the initial time is set to $t_0 = 80\mathrm{s}$. At this time, the ground-truth altitude and vertical velocity are $h(t_0)=140.64\ \mathrm{m}$ and $v_h(t_0) = -1.43\ \mathrm{m/s}$, respectively. The ground-truth attitude $\bar R (t_0)$ corresponds to the Euler angles $(\mathrm{yaw},\mathrm{pitch},\mathrm{roll}) = (138.44 \mathrm{^\circ},-3.72 \mathrm{^\circ}, 11.43 \mathrm{^\circ})$. The corresponding ground-truth tilt direction is $\bar z(t_0) = \bar R(t_0)^{\top}\mathrm{e}_3 = [0.065\,\,0.2\,\,0.98]^{\top}$.
A Monte Carlo simulation with $50$ runs is performed. In each run, the initial estimates are randomly sampled from Gaussian distributions. For the LES observer, the initial altitude and vertical velocity estimates are centered at $\hat h^L (t_0) =h(t_0) + 10\ \mathrm{m} $ and $\hat v_h^L (t_0) =v_h(t_0)+ 3\ \mathrm{m/s}$ with standard deviations of $8\ \mathrm{m}$ and $8\ \mathrm{m/s}$, respectively. The initial attitude estimate $\hat R^L(t_0)$ is generated from $(\mathrm{yaw},\mathrm{pitch},\mathrm{roll}) =(60^\circ,-30^\circ, 45^\circ)$ with a standard deviation of $60^\circ$ applied independently to each axis. For the AGAS observer, the initial attitude estimate $\hat R^G (t_0)$ is generated using the same nominal Euler angles $(60^\circ,-30^\circ, 45^\circ)$, with independent perturbations of a $60^\circ$ standard deviation on each axis. The initial state estimates is centered at $\hat x^G(t_0) =\begin{bmatrix} \hat h^L (t_0)&\hat v_h^L (t_0)&\hat z(t_0)\end{bmatrix}^\top$ with standard deviation $\begin{bmatrix} 8&8&0.5&0.5&0.5\end{bmatrix}^\top$, where $\hat z(t_0) = \hat R^G(t_0)^{\top}\mathrm{e}_3 = [0.41\,\,-0.91\,\,-0.024]^{\top}$.
The Riccati matrices are initialized according to the initial estimation uncertainty. Specifically, \[P^L(t_0) = \operatorname{diag}\left(\tilde h^L(t_0)^2,\tilde v_h^L(t_0)^2,\sigma_{\lambda}(t_0)^2, \sigma_{\lambda}(t_0)^2,\sigma_{\lambda}(t_0)^2\right),\] and \[P^G(t_0) = \operatorname{diag}\left(\tilde h^L(t_0)^2,\tilde v_h^L(t_0)^2,\sigma_z(t_0)^2, \sigma_z(t_0)^2,\sigma_z(t_0)^2\right),\] where \[\sigma_{\lambda}(t_0) =\frac{|\tilde \lambda (t_0)|}{\sqrt{3}}, \qquad \sigma_z(t_0) = \frac{|\bar z(t_0) - \hat z(t_0)|}{\sqrt{3}}.\]
The IMU measurements $\mathrm{a}(t)$ and $\mathrm{\omega}(t)$ are discretized at $250\ \mathrm{Hz}$. The body-frame magnetometer measurement $\mathrm{m}_{\mathcal{B}}(t)$ is available at $50\ \mathrm{Hz}$ with estimated standard deviation $\sigma_m = 10^{-3}[10.6\,\,7.5\,\,11.8]^\top$ yielding the measurement covariance matrix $Q_m := \operatorname{diag}(\sigma_{m,x}^2,\sigma_{m,y}^2, \sigma_{m,z}^2)$. The barometric altitude measurement is sampled at $5\ \mathrm{Hz}$ with estimated standard deviation $\sigma_b = 0.13\ \mathrm{m}$, resulting the scalar covariance $Q_b := \sigma_b^2$. The observer gains and covariance parameters are selected as $S^L=\operatorname{diag}(1,1,0.1,0.1,0.1)$, $Q^L = \operatorname{diag}(Q_b,Q_m)$for the LES observer.
For the AGAS observer $S^G=S^L$, $Q^G=Q_b$, together with $k_z = 8$ and $k_m = 1$.

The results presented in Fig.~\ref{fig:Flight_Euler_Angles_Att_Error_large} and Table~\ref{tab:flight_perf_comparison} demonstrate that the attitude estimates $\hat R^G$ and $\hat R^L$ remain bounded and converge towards the ground-truth attitude $\bar R$ despite the large initialization errors introduced at $t_0 = 80\mathrm{s}$. This behavior is reflected by the bounded yaw, pitch, and roll estimates and progressive reduction of the corresponding attitude estimation errors $e_{att}$ throughout the experiment. Such an outcome is non-trivial given the magnitude of the initial perturbations, which place both estimators far from the true state and therefore constitute a demanding validation scenario. Nevertheless, clear and structured differences emerge between the two approaches across all the phases of the estimation process.
During the transient phase, LES observer exhibits a shorter convergence time, lower peak attitude error, and smaller full-segment RMSE values across all Euler-angle channels. These results are consistent with the respective stability properties of the two proposed observers predicted by Theorems~\ref{Theorem1} and \ref{Theorem2}. The LES architecture is designed to achieve local exponential convergence and therefore promotes rapid contraction of the estimation errors. In contrast, the AGAS observer is designed to provide almost-global convergence convergence from arbitrary initial attitude estimates. Consequently, its correction mechanism prioritizes robustness over a large portion of the rotation manifold rather than maximizing local convergence speed. The slower transient response of AGAS can therefore be interpreted as the counterpart of its strong almost-global convergence guarantees.
The observability analysis further explains the performance differences observed during initialization. As shown in Fig.~\ref{fig:Flight_PE_Condition}, both observers experience poorest conditioning over $t \in [82,\,90]\mathrm{s}$, where $\log\left(\mathrm{cond}(M)\right)>4$, with $M := M^L$ for LES and $M:= M^G$ for AGAS, corresponding to condition numbers on the order of $10^4$. Such values indicate a highly ill-conditioned estimation problem and a weakly satisfied persistent excitation conditions. Under these circumstances, the relaxed observability condition associated with the LES observer becomes particularly advantageous, enabling more effective exploitation of the available measurement information when excitation is limited.
As the flight progresses over $t \in [90,\,110]\mathrm{s}$, successive maneuvers enrich the excitation and substantially improve the conditioning of the estimation problem. The condition numbers of both estimators decrease accordingly, and the performance gap progressively narrows. Both observers converge to a small neighborhood of zero attitude error and maintain accurate attitude reconstruction throughout the remainder of the flight segment.
The steady-state metrics indicate that the AGAS observer achieves a slightly lower aggregate attitude error and improved roll and yaw RMSE values, while LES retains a modest advantage in pitch. These results suggest that, once sufficient excitation is available and both observers have converged, the estimation performance becomes primarily governed by the respective correction structures of the two architecture rather than by observability limitations.
Overall, the experimental results provide strong empirical support for the theoretical developments of this work. In particular, the relaxed observability condition of the LES observer translates into measurable improvements in convergence speed and transient accuracy even under limited excitation, while the AGAS demonstrates reliable recovery from large initialization errors together with competitive steady-state performance. These findings confirm that the distinct stability and observability properties of the two observers are directly reflected in their practical behaviors under real-flight conditions.

\begin{figure}[!t]
    \centering
    \includegraphics[width=\linewidth]{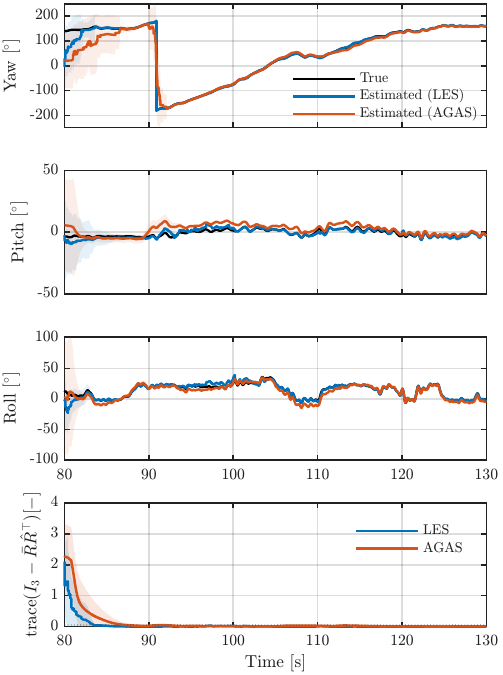}
    \caption{Results using the real flight data. Euler angles and attitude error under large initial values.}
    \label{fig:Flight_Euler_Angles_Att_Error_large}
\end{figure}

\begin{figure}[!t]
    \centering
    \includegraphics[width=\linewidth]{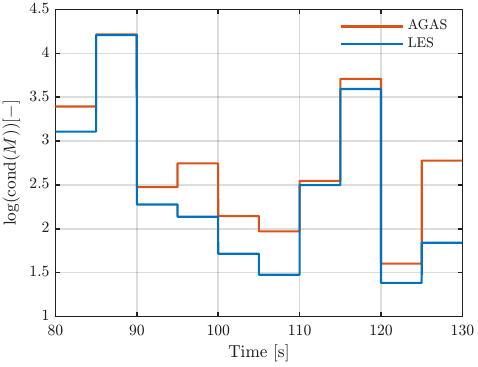}
    \caption{Evolution of cond$(M^L)$ for LES-based observer and cond$(M^G)$ for AGAS-based observer over the time from the real flight data.}
    \label{fig:Flight_PE_Condition}
\end{figure}

\begin{table}[t]
\centering
\caption{Performance Comparison between LES and AGAS observers from flight data}
\label{tab:flight_perf_comparison}
\begin{tabular}{lccc}
\hline
\textbf{Metric} & \textbf{LES} &  \textbf{AGAS} & Improvement ($\%$)\\
\hline
$t_c$ $[\mathrm{s}]$ &\textbf{$83$} & $88$ &$+5.6$\\
$e_{\mathrm{att,\max}}[-]$ &\textbf{$1.98$} &$2.23$ &$+11.2$\\
$e_{\mathrm{att,ss}}[-]$ &$0.0015$ &\textbf{$0.0014$} &$-7.1$\\
RMSE$_{\phi}$$[^\circ]$ &\textbf{$1.5$} &$2.2$ &$+32$\\
RMSE$_{\theta}$$[\circ]$ &\textbf{$2$} &$3$ &$+33.3$\\
RMS$_{\psi}$$[^\circ]$ &\textbf{$3$} &$5$ &$+40$\\
RMSE$_{\phi,\mathrm{ss}}$$[^\circ]$ &$0.4$ &\textbf{$0.25$} &$-60$\\
RMSE$_{\theta,\mathrm{ss}}$ $[^\circ]$ &\textbf{$0.5$} &$0.8$ &$+37.5$\\
RMSE$_{\psi,\mathrm{ss}}$$[^\circ]$& $0.9$&\textbf{$0.56$} & $-60$\\
\hline
\end{tabular}
\end{table}

\section{Conclusion}\label{sec:concl}
This paper investigated barometer-aided attitude estimation as a lightweight sensing alternative for accelerating vehicles operating in environments where conventional velocity-aided attitude estimation is infeasible due to the lack of reliable velocity measurements. By exploiting the information contained in barometric altitude measurements and their coupling with the vehicle dynamics, we established a theoretical framework for attitude estimation using only inertial, magnetic, and barometric sensors. This work provides the first unified treatment of barometer-aided attitude estimation encompassing observability analysis, observer design with both LES and AGAS guarantees, and validation on real-flight data.

Within this framework, two observer architectures with complementary convergence and robustness properties were proposed. The AGAS architecture combines a deterministic Riccati-based estimation stage with a nonlinear observer on $\mathrm{SO}(3)$, yielding almost-global asymptotic convergence under a uniform observability condition while preserving the geometric structure of the attitude dynamics. The LES architecture employs a reduced-order nonlinear observer on $\mathrm{SO}(3)\times\mathbb{R}^2$ and achieves local exponential convergence under a relaxed observability condition requiring sufficient excitation only in the horizontal component of the inertial acceleration.

The proposed observers were validated through extensive Monte Carlo simulations and real-flight experiments. The results show that the LES observer provides faster convergence and improved estimation accuracy under limited excitation, whereas the AGAS observer exhibits stronger robustness to large initialization errors while maintaining competitive steady-state performance. These observations are consistent with the theoretical analysis and highlight the tradeoff between convergence rate, excitation requirements, and domain of attraction associated with the two observer designs. Future work will focus on extending the proposed architectures to account for IMU biases and on investigating the resulting observability and stability properties.

\appendices
\section{Proof of Lemma~\ref{Lemma1}}\label{appendix_A}
To show that the system ~\eqref{eq:ltv_G} is uniformly observable, it suffices to show that there exist \( \bar{\delta}, \bar{\mu} > 0 \) such that~\eqref{eq:observability_gramian} holds, i.e. \(W^G(t, t + \bar \delta) \ge \bar \mu I_5,~ \forall t \geq 0,\) with \(W^G(t, t + \bar \delta)\) given by~\eqref{eq:gramian_G}.
Assume, by contradiction, that system~\eqref{eq:ltv_G} is not uniformly observable.
Then, for every \(\bar \mu > 0\) and  \(\bar \delta > 0\), there exists \(t \ge 0\) such that \(W^G(t,t+\bar\delta) < \bar\mu I_5 \), 
Let \(\{\mu_p\}_{p\in\mathbb{N}}\) be a sequence decreasing to zero with \(\mu_p > 0\), and let 
\(\bar{\delta} > 0\) satisfy the PE condition~\eqref{eq:PE_Condition_G}.  
Then, there exist sequences \(\{t_p\} \subset \mathbb{R}_+ \) and \(\{d_p\} \subset \mathbb{S}^4\),  
      such that 
\(
d_p^\top W^G(t_p,t_p+\bar{\delta}) d_p < \mu_p,
\quad \forall p \in \mathbb{N}.
\)
By compactness of \(\mathbb{S}^4\), there exists a subsequence of \(\{d_p\}\) that converges to some \(d \in \mathbb{S}^4\).
Letting \(p \to \infty\) and using \(\mu_p \to 0\) yields the convergence relation from  \eqref{eq:gramian_G}
\begin{equation}
\lim_{p \to \infty} 
\int_0^{\bar{\delta}} 
\big\| C^G\, \Phi^G(s,t_p)\, d_p \big\|^2 ds = 0 \label{eq:conv_relation_G}
\end{equation}
or, by a change of variables with a scalar output function,
\begin{equation}
\lim_{p \to \infty} 
\int_0^{\bar{\delta}} |f_p(s)|^2 ds = 0 \label{eq:conv_relation_f_G}
\end{equation}
where we define
\[
\begin{aligned}
f_p(t) &:= C^G \Phi^G(t+t_p,t_p)d_p \\
&= C^G_1 \phi_{11}(t+t_p,t_p) d_1 + C_1^G \phi_{12}(t+t_p,t_p) d_2,
\end{aligned}
\]
where \(d_p = d = [\,d_1^\top, d_2^\top\,]^\top\) with 
\(d_1 \in \mathbb{R}^2\), \(d_2 \in \mathbb{R}^3\), and $|d_1|^2+|d_2|^2=1$.
From~\eqref{eq:transition_matrix},~\eqref{eq:state_matrix_G}-\eqref{eq:state_trans_matrix_G},~\eqref{eq:conv_relation_G}-\eqref{eq:conv_relation_f_G}, and knowing that \((A^G_{11})^2 = 0_{2\times2}\), we obtain the successive time-derivatives of \(f_p(t)\), as follow :
\[
\begin{aligned}
f_p^{(1)}(t) &= C^G_1 A^G_{11} \phi_{11} d_1 
             + C^G_1 \left( A^G_{11}\phi_{12} + A^G_{12}\phi_{22} \right) d_2, \\
f_p^{(2)}(t) &= C^G_1 \left( A^G_{11} A^G_{12} - A^G_{12}\mathrm{\omega}^\times \right) \phi_{22} d_2
\end{aligned}
\]
Now, using the results of Lemma~A.1 of \cite{Pascal2017},  we deduce:
\[
\lim_{p \to \infty} 
\int_0^{\bar{\delta}} |f_p^{(k)}(s)|^2 ds = 0, 
\quad k=0,1,\dots,2.
\]
The highest derivative yields 
\[f_p^{(2)}(t_p) \to \mathrm{a}^\top(t_p)\big(-\mathrm{\omega}^\times(t_p)\big) d_2 \to 0 ~\text{as } p \to \infty.\]
By the PE assumption in~\eqref{eq:PE_Condition_G}, this implies \(d_2 = 0\).  
Substituting \(d_2 = 0\) into \(f_p(s)\) and \(f_p^{(1)}(s)\) gives 
\begin{equation} 
C^G_1\phi_{11}d_1 = 0,  \quad C^G_1 A_{11}\phi_{11}d_1 = 0, \label{eq:fp_fp_1_G}
\end{equation}
with \(d_1 = \begin{bmatrix}d_{1,1},d_{1,2}\end{bmatrix}^\top,\) where \(d_{1,1} \in \mathbb{R}\) and \(d_{1,2} \in \mathbb{R}.\)
From~\eqref{eq:phi_11} we compute \(C^G_1 \phi_{11}(t,\tau) =  \begin{bmatrix}1 & (t-\tau)\end{bmatrix}\) and \(C^G_1 A^G_{11}\phi_{11}(t,\tau) =  \begin{bmatrix}0 & 1\end{bmatrix}\).
By substituting these into~\eqref{eq:fp_fp_1_G}, we get \(
d_{1,1}+d_{1,2}(t-\tau) = d_{1,2}  = 0\). This implies \(d_1 = 0\).
Therefore \(d = 0\), contradicting \(|d| = 1\).  
Hence, the pair \((A^G(t),C^G)\) is uniformly observable. This in turn guarantees the global exponential stability of the equilibrium \(\tilde{x}^G = \boldsymbol{0}_{5\times 1}\)(See~\cite{Hamel2017PositionMeasurements}).

\section{Proof of Theorem~\ref{Theorem1}}\label{appendix_B}
Differentiating \( \tilde{R}^G \) and using the tilt error \( \tilde{z} := z -\hat z\), we obtain
\begin{align}
\dot{\tilde{R}}^G &=  R\mathrm{\omega}^\times (\hat R^G)^\top + R\left(-\mathrm{\omega}^\times (\hat R^G)^\top + (\hat R^G)^\top \sigma_R^\times \right) \notag \\
&= \tilde{R}^G\left( k_z \mathrm{e}_3 \times \hat{R}^G \hat{z} + k_m \bar{\mathrm{m}}_{\mathcal{I}} \times \hat{R}^G \bar{\mathrm{m}}_{\mathcal{B}} \right)^\times
\notag 
\\&= \tilde{R}^G\big(k_z\mathrm{e}_3 \times \hat{R}^Gz- k_z\mathrm{e}_3 \times \hat{R}^G \tilde{z} +k_m \bar{\mathrm{m}}_{\mathcal{I}} \times \hat{R}^G \bar{\mathrm{m}}_{\mathcal{B}}\big)^\times \notag\end{align}

From Lemma~\ref{Lemma1}, it follows that \( \hat z \to z\) exponentially, which implies that \( \bar{\mathrm{m}}_{\mathcal{B}} \to \bar{\bar{\mathrm{m}}}_{\mathcal{B}} \), with \( \bar{\bar{\mathrm{m}}}_{\mathcal{B}} = \bar{\Pi}_{z} \mathrm{m}_{\mathcal{B}} \). Moreover, one can show that \(\bar{\mathrm{m}}_{\mathcal{B}} = \bar{\bar{\mathrm{m}}}_{\mathcal{B}} + \mathcal{O}(\tilde{z}).\)
Expressing this in terms of \(\tilde x \), and in view of ~\eqref{eq:riccati_observer_G}, one obtains the closed-loop system:
\begin{subequations} \label{eq:closed_loop}
\begin{align}
\dot{\tilde{R}}^G &=  \tilde{R}^G \left(
k_z \mathrm{e}_3 \times \hat{R}^Gz 
+ k_m \bar{\mathrm{m}}_{\mathcal{I}} \times \hat{R}^G \bar{\bar{\mathrm{m}}}_{\mathcal{B}} 
+ \mathcal{O}(\tilde x^G) \right)^\times, \label{eq:closed_loop_att} \\
\dot{\tilde{x}}^G &= (A^G - K^GC^G)\tilde{x}^G. \label{eq:closed_loop_trans}
\end{align}
\end{subequations}
The above system can be seen as a cascade interconnection of a non-linear system on \(\mathrm{SO}(3)\)~\eqref{eq:closed_loop_att} and the LTV system on \(\mathbb{R}^5\)~\eqref{eq:closed_loop_trans}. To prove the AGAS of the interconnection system, we begin by proving that subsystem~\eqref{eq:closed_loop_att} is AGAS for \(\tilde{x}^G = 0_{5\times1}\).
From~\eqref{eq:closed_loop_att}, it follows, as shown in~\cite{mahony2008nonlinear}, that the equilibrium sets are 
\(\mathcal{E}_s = \{I_3\}\) and \(\mathcal{E}_u = \left\{(U\Lambda U^\top, 0)\,\middle|\, \Lambda = \operatorname{diag}(1, -1, -1),\, U \in \mathrm{SO}(3) \right\}.\)
The singleton set \(\mathcal{E}_s\) is the stable equilibrium, and the set \(\mathcal{E}_u\) is the set of unstable equilibria (see~\cite[Th.~6.1]{VanGoor2025}). It consists of all 180-degree rotations, each defined by an axis on \(S^2\), which corresponds to a 2D space embedded in the 3D manifold \(\mathrm{SO}(3) \), and thus has measure zero in \(\mathrm{SO}(3) \). It follows that the stable equilibrium \(\tilde{R}^G = I_3\) is almost globally asymptotically stable for subsystem~\eqref{eq:closed_loop_att}. To complete the proof, we now examine the full interconnection
system. Since the estimation error \(\tilde x^G\) in~\eqref{eq:closed_loop_trans} evolves independently of \(\tilde{R}^G\) and is GES from Lemma~\ref{Lemma1}, there exist constants \(\delta, \beta > 0\) such that \(\tilde x^G\) satisfies \(
|\tilde x^G(t)| \leq \delta \exp(-\beta t)\, |\tilde x^G(0)|, \forall t \geq 0.
\)
Thus, \(\tilde x^G\) remains uniformly bounded, meaning there exists a compact set \(S \subset \mathbb{R}^5\) such that \(\tilde x^G(t) \in S\) for all \(t \geq 0\). Therefore, according to~\cite[Proposition~2]{Angeli2010}, one can conclude that subsystem~\eqref{eq:closed_loop_att} is almost globally Input-to-State Stable (ISS) with respect to \(\tilde{R}^G = I_3\) and input \(\tilde x^G\). Hence, given that \(\tilde x^G = 0_{5\times 1}\) for system~\eqref{eq:closed_loop_trans} is GES and that subsystem~\eqref{eq:closed_loop_att} with \(\tilde x^G = 0_{5\times 1}\) is AGAS at \(\tilde{R}^G = I_3\) and almost globally ISS with respect to \(\tilde x^G\), it follows from~\cite[Th.~2]{Angeli2004} that the cascaded interconnection system~\eqref{eq:closed_loop} is AGAS at \((\tilde{R}^G, \tilde x^G) = (I_3,0_{5\times 1})\).

\section{Proof of Lemma~\ref{Lemma2}} \label{appendix_C}
To establish the uniform observability, we evaluate the observability Gramian in a convenient block form. To this end, the system matrix $A^{L\star}(t)$ is partitioned as
\begin{equation}
A^{L\star}(t) =\begin{bmatrix}A^L_{11} &A^{L\star}_{12}(t) \\
 0_{3\times 2} &  0_{3\times 3}
\end{bmatrix}, \label{eq:block_true_state_matrix}
\end{equation}
with
\[
A^L_{11} = 
\begin{bmatrix}
0 & 1 \\
0 &  0
\end{bmatrix}, \quad
A^{L\star}_{12}(t) = \begin{bmatrix}0_{1\times 3}\\ -\mathrm{e}_{3}^{\top}(R \mathrm{a})^\times \end{bmatrix}.
\]
Similarly, the output matrix $C$ is partitioned as
\begin{equation}
C^L = \begin{bmatrix} C^L_{11} & 0_{1\times 3}\\
0_{3\times 2} & C^L_{22}
\end{bmatrix}, \label{eq:block_true_output_matrix}
\end{equation}
with
\[C^L_{11} = \begin{bmatrix}1&0\end{bmatrix}, \quad
C^L_{22}=-(\mathrm{m}_{\mathcal{I}})^{\times}.
\]
Due to the structure of \(A^{L\star}\) in~\eqref{eq:block_true_state_matrix}, the associated state transition matrix admits the form:
\begin{equation}
\Phi^{L\star}(t,\tau) =
\begin{bmatrix} 
\Phi_{11}(t,\tau) &\Phi^{\star}_{12}(t,\tau)\\
0_{3\times 2} & I_3
\end{bmatrix},\label{eq:true_state_trans_matrix}
\end{equation}
with \(\Phi_{11}\in\mathbb{R}^{2 \times 2}\) and \( \Phi^{\star}_{12}\in\mathbb{R}^{2\times3}\).
Substituting~\eqref{eq:block_true_state_matrix} and \eqref{eq:true_state_trans_matrix} into the state transition equation~\eqref{eq:transition_matrix} yields
\begin{equation}
\begin{aligned}
\frac{d}{dt} \Phi^{L\star}(t, \tau)
&=\begin{bmatrix}
A_{11}\Phi_{11} & A_{11}\Phi^{\star}_{12}+A^{\star}_{12}  \\
0_{3\times 2} & 0_{3\times 3}\\
\end{bmatrix}, \Phi^{L\star}(\tau,\tau) = I_5, \label{eq:deriv_true_trans_matrix}
\end{aligned}
\end{equation}
with initial conditions \(\Phi_{11}(\tau,\tau) = I_2\) and \(\Phi^{\star}_{12}(\tau,\tau) = 0_{2 \times 3}\).
This directly leads to the subsystem equations 
\begin{subequations}
\begin{align}
\dot \Phi_{11} &= A^L_{11}\Phi_{11} \label{eq:dot_phi11},\\
\dot \Phi^{\star}_{12} &= A^L_{11}\Phi^{\star}_{12}+A^{L\star}_{12} \label{eq:dot_phi12}.
\end{align}
\end{subequations}
Since \(A^L_{11}\) is a constant matrix, the solution of~\eqref{eq:dot_phi11} is given by 
\begin{equation}
\Phi_{11}(t,\tau) = \exp\left(A^L_{11}(t-\tau)\right) = \begin{bmatrix} 1 & (t-\tau) \\ 0 & 1 \end{bmatrix}\label{eq:Phi_11}.
\end{equation}
The dynamics in~\eqref{eq:dot_phi12} define a linear time varying system with constant state matrix $A^L_{11}$. Using the integral representation of the solution, one obtains
\begin{equation}
\Phi^{\star}_{12}(t,\tau) = \Phi_{11}(t,\tau) \Phi^{\star}_{12} (\tau,\tau)
+\int^{t}_{\tau}\Phi_{11}(t,s) A^{L\star}_{12}(s)\ ds. \label{eq:eq_ODE_Phi_12_a}
\end{equation}
Using \(\Phi^{\star}_{12}(\tau,\tau) = 0_{3 \times 3}\), \eqref{eq:eq_ODE_Phi_12_a} reduces to
\begin{equation}
\Phi^{\star}_{12}(t,\tau) = \int^{t}_{\tau}\Phi_{11}(t,s)A^{L\star}_{12}(s)\, ds. \label{eq:eq_ODE_Phi_12_b}
\end{equation}
Substituting \(A^{L\star}_{12}\) and~\eqref{eq:Phi_11} in~\eqref{eq:eq_ODE_Phi_12_b} yields:
\begin{equation}
\Phi^{\star}_{12}(t,\tau) =\begin{bmatrix}
-\int_{\tau}^{t}\left(t-s\right)\mathrm{e}_3^{\top}(R\mathrm{a})^{\times}\,ds \\
-\int_{\tau}^{t}-\mathrm{e}_3^{\top}(R\mathrm{a})^{\times}\,ds
\end{bmatrix}. \label{eq:Phi_12} 
\end{equation}
Next, from~\eqref{eq:block_true_output_matrix} and \eqref{eq:true_state_trans_matrix}, \(C\Phi^{L \star}\) can be written as
\begin{equation}
C^L\Phi^{L\star}(s,t)=
\begin{bmatrix}
C^L_{11}\Phi_{11}(s,t) & C^L_{11}\Phi^{\star}_{12}(s,t)\\
0_{3\times 2} & C^L_{22}
\end{bmatrix}\label{eq:C_Phi}
\end{equation}
Substituting~\eqref{eq:C_Phi} into the observability Gramian definition~\eqref{eq:W}, one obtains 
\begin{equation}
W^{L\star}(t,t+\tau)=
\begin{bmatrix}
W^L_{11}(t,t+\tau) & W_{12}^{L\star}(t,t+\tau)\\
W_{12}^{L\star \top}(t,t+\tau) & W_{22}^{L\star}(t,t+\tau)
\end{bmatrix} \label{eq:gramian}
\end{equation}
where 
\[
W^L_{11}(t, t+\tau) = \frac{1}{\tau}\int_t^{t+\tau}
\Phi_{11}^{\top}(s,t) (C^L_{11})^{\top} C^L_{11}\Phi_{11}(s,t)\,ds,   
\]
\[
W^{L\star}_{12}(t,t+\tau) = \frac{1}{\tau}\int_t^{t+\tau}\Phi_{11}^{\top}(s,t) (C^L_{11})^{\top} C^L_{11}\Phi^{\star}_{12}(s,t)\,ds,
\]
and
\[
\begin{aligned}
W^{L\star}_{22}(t,t+\tau) &= \frac{1}{\tau}\int_t^{t+\tau}\Phi^{\star \top}_{12}(s,t) (C^L_{11})^{\top} C^L_{11}\Phi^{\star}_{12}(s,t)\,ds \\
&+\frac{1}{\tau}\int_t^{t+\tau}(C^L_{22})^{\top}C_{22}\,ds.
\end{aligned}
\]

To show that there exist $\bar \delta,\bar \mu >0$ such that \(W^{L\star}(t, t + \bar \delta) \ge \bar \mu I_5,\, \forall t \geq 0,\) we proceed by contradiction following the same arguments as in the proof of Lemma~\ref{Lemma1} given in Appendix~\ref{appendix_A}.
Letting \(p \to \infty\) and using \(\mu_p \to 0\) yields the convergence relation from  \eqref{eq:gramian}
\begin{equation}
\lim_{p \to \infty} 
\int_0^{\bar{\delta}} 
\big\| C^L\, \Phi^{L\star}(s,t_p)\, d_p \big\|^2 ds= 0 \label{eq:conv_relation_L}
\end{equation}
or, by a change of variables with an output function,
\begin{equation}
\lim_{p \to \infty} 
\int_0^{\bar{\delta}} \big|f_p(s)\big|^2 ds = 0, \label{eq:conv_relation_f_L}
\end{equation}
where we define
\[
\begin{aligned}
f_p(t) &:= C^L \Phi^{L \star}(t+t_p,t_p) d_p\\
&= \begin{bmatrix}C^L_{11} \Phi_{11}(t+t_p,t_p) d_1+C_{11}\Phi^{\star}_{12}(t+t_p,t_p)d_2\\C^L_{22}d_2
\end{bmatrix} \\
&= \begin{bmatrix} f_{p,1}(t)\\f_{p,2}(t)\end{bmatrix}.
\end{aligned}
\] 
This implies
\[
\begin{aligned}
|f_p(t)|^2 &= \big| C^L_{11}\Phi_{11}d_1 +C^L_{11}\Phi^{\star}_{12}d_2\big|^2 + \big| C^L_{22}d_2 \big|^2\\
&=\big|f_{p,1}(t)\big|^2 + \big|f_{p,2}(t)\big|^2.
\end{aligned}
\]
Hence the convergence relation in~\eqref{eq:conv_relation_f_L} becomes
\begin{subequations}
\begin{align}
&\lim_{p \to \infty} 
\int_0^{\bar{\delta}} 
\big|f_{p,1}(s) \big|^2 ds = 0, \text{ and} \label{eq:conv_baro}\\
&\lim_{p \to \infty} 
\int_0^{\bar{\delta}} 
\big|f_{p,2}(s) \big|^2 ds = 0. \label{eq:conv_magneto}
\end{align}
\end{subequations}
From~\eqref{eq:block_true_state_matrix}-\eqref{eq:Phi_12}, we obtain the successive time derivatives of \(f_{p,1}(t)\), as follow :
\[
\begin{aligned}
f_{p,1}^{(1)}(t) &= C^L_{11} A^L_{11} \Phi_{11} d_1 
             + C^L_{11}\left( A^L_{11}\Phi^{\star}_{12} + A^{L\star}_{12} \right) d_2,\\
f_{p,1}^{(2)}(t) &=C^L_{11}\left(A^L_{11}A^{L\star}_{12}+\dot A^{L\star}_{12}\right)d_2 = -\mathbf{e}^{\top}_3(R\mathrm{a})^{\times}d_2.
\end{aligned}
\]
Using the results of Lemma~A.1 of \cite{Pascal2017},  we deduce:
\[
\lim_{p \to \infty} 
\int_0^{\bar{\delta}} |f_{p,1}^{(k)}(s)|^2 ds = 0, 
\quad k=0,1,2..
\]
Letting \(d_2 = \begin{bmatrix}d^{\top}_{2,1},d_{2,2}\end{bmatrix}^\top\), with \(d_{2,1} \in \mathbb{R}^2\) and \(d_{2,2} \in \mathbb{R}\), the highest derivative in the tangent-space yields
\[f_{p,1}^{(2)}(t_p) \to -\mathbf{e}^{\top}_3\left( R(t_p) \mathrm{a}(t_p)\right)^{\times}Jd_{2,1} \to 0 \text{ as }  p \to \infty.\] By the PE assumption in~\eqref{eq:PE_condition}, this leads to \(d_{2,1} = 0\), which implies \(d_2 = \begin{bmatrix}0_{1 \times 2} ,d_{2,2}\end{bmatrix}^\top.\) Substituting \(d_2 = d_{2,2}\mathrm{e}_3\) into $f_{p,2}(t_p)$, we get $d_{2,2}(\mathrm{m}_\mathcal{I}\times \mathrm{e}_3) \to 0 \text{ as }  p \to \infty.$
By assumption, vectors $\mathrm{m}_\mathcal{I}$ and $\mathrm{e}_3$ are non-collinear, then $d_{2,2}\to 0 \text{ as }  p \to \infty.$ This implies $d_2 = 0.$
Now, substituting $d_2 =0$ into $f_{p,1}(t_p)$ and $f_{p,1}^{(1)}(t_p)$, we get 
\begin{equation} 
C^L_{11}\Phi_{11}d_1 = 0,  \quad C^L_{11} A^L_{11} \Phi_{11} d_1 = 0, \label{eq:fp_fp_1}
\end{equation}
with \(d_1 = \begin{bmatrix}d_{1,1},d_{1,2}\end{bmatrix}^\top,\) where \(d_{1,1} \in \mathbb{R}\) and \(d_{1,2} \in \mathbb{R}.\)
From~\eqref{eq:Phi_11} we have \(C^L_{11} \Phi_{11}(t,\tau) =  \begin{bmatrix}1 & (t-\tau)\end{bmatrix}\) and \(C^L_{11}A^L_{11}\Phi_{11}(t,\tau) =  \begin{bmatrix}0 & 1\end{bmatrix}.\)
By substituting these into~\eqref{eq:fp_fp_1}, we get $d_{1,2} \to 0$ and \(
d_{1,1}+d_{1,2}(t_p-\tau) \to  d_{1,1}  \to  0 \text{ as }  p \to \infty.\) This implies \(d_1 = 0\). Hence $d = 0$, contradicting $|d| = 1$. Therefore, the pair $(A^{L\star}(t),C)$ is uniformly observable.

\section*{Acknowledgment}
This work was supported by the "Grands Fonds Marins" Project Deep-C, and the ASTRID ANR project ASCAR. This research work is also supported in part by NSERC-DG RGPIN-2020-04759 and Fonds de recherche du Québec (FRQ).

\bibliographystyle{IEEEtran}
\bibliography{references}

@inproceedings{roberts2011attitude,
  title={On the Attitude Estimation of Accelerating Rigid-Bodies Using {GPS} and {IMU} Measurements},
  author={Roberts, Andrew and Tayebi, Abdelhamid},
  booktitle={50th IEEE Conference on Decision and Control and European Control Conference},
  pages={8088--8093},
  year={2011},
  organization={IEEE}
}

@article{hua2010attitude,
  title={Attitude Estimation for Accelerated Vehicles Using {GPS/INS} Measurements},
  author={Hua, Minh-Duc},
  journal={Control Engineering Practice},
  volume={18},
  number={7},
  pages={723--732},
  year={2010},
  publisher={Elsevier}
}

@inproceedings{berkane2017attitude,
  title={Attitude and Gyro Bias Estimation Using {GPS} and {IMU} Measurements},
  author={Berkane, Soulaimane and Tayebi, Abdelhamid},
  booktitle={56th Annual Conference on Decision and Control (CDC)},
  pages={2402--2407},
  year={2017},
  organization={IEEE}
}

@article{Hamel2017PositionMeasurements,
    title = {Position Estimation from Direction or Range Measurements},
    year = {2017},
    journal = {Automatica},
    author = {Hamel, Tarek and Samson, Claude},
    pages = {137--144},
    volume = {82},
    publisher = {Elsevier Ltd},
    doi = {10.1016/j.automatica.2017.04.045},
    issn = {00051098}
}

@article{hamel2017riccati,
  title={Riccati observers for the nonstationary PnP problem},
  author={Hamel, Tarek and Samson, Claude},
  journal={IEEE Transactions on Automatic Control},
  volume={63},
  number={3},
  pages={726--741},
  year={2017},
  publisher={IEEE}
}

@article{2008_bonnabel_SymmetryPreservingObservers,
  title = {Symmetry-Preserving Observers},
  author = {Bonnabel, Silv{\`E}re and Martin, Philippe and Rouchon, Pierre},
  year = {2008},
  month = dec,
  journal = {IEEE Transactions on Automatic Control},
  volume = {53},
  number = {11},
  pages = {2514--2526},
  issn = {1558-2523},
  doi = {10.1109/TAC.2008.2006929},
  }

@article{2008_martin_InvariantObserverEarthVelocityAided,
  title = {An Invariant Observer for Earth-Velocity-Aided Attitude Heading Reference Systems},
  author = {Martin, Philippe and Sala{\"u}n, Erwan},
  year = {2008},
  month = jan,
  journal = {IFAC Proceedings Volumes},
  series = {17th {{IFAC World Congress}}},
  volume = {41},
  number = {2},
  pages = {9857--9864},
  issn = {1474-6670},
  doi = {10.3182/20080706-5-KR-1001.01668},
  langid = {english},
}

@article{2016_hua_StabilityAnalysisVelocityaided,
  title = {Stability Analysis of Velocity-Aided Attitude Observers for Accelerated Vehicles},
  author = {Hua, Minh-Duc and Martin, Philippe and Hamel, Tarek},
  year = {2016},
  month = jan,
  journal = {Automatica},
  volume = {63},
  pages = {11--15},
  issn = {0005-1098},
  doi = {10.1016/j.automatica.2015.10.014},
  langid = {english}
}

@inproceedings{2013_troni_PreliminaryExperimentalEvaluation,
  title = {Preliminary Experimental Evaluation of a {Doppler}-Aided Attitude Estimator for Improved Doppler Navigation of Underwater Vehicles},
  booktitle = {{{IEEE International Conference}} on {{Robotics}} and {{Automation}}},
  author = {Troni, Giancarlo and Whitcomb, Louis L.},
  year = {2013},
  month = may,
  pages = {4134--4140},
  issn = {1050-4729},
  doi = {10.1109/ICRA.2013.6631160}
}

@article{Pieter2023,
title = {Constructive Equivariant Observer Design for Inertial Velocity-Aided Attitude},
journal = {IFAC-PapersOnLine},
volume = {56},
number = {1},
pages = {349-354},
year = {2023},
note = {12th IFAC Symposium on Nonlinear Control Systems NOLCOS 2022},
issn = {2405-8963},
author = {Pieter {van Goor} and Tarek Hamel and Robert Mahony},
}

@ARTICLE{benallegue2023velocity,

  author={Benallegue, Mehdi and Benallegue, Abdelaziz and Cisneros, Rafael and Chitour, Yacine},

  journal={IEEE Transactions on Automatic Control}, 

  title={Velocity-Aided {IMU}-Based Tilt and Attitude Estimation}, 

  year={2023},

  volume={68},

  number={10},

  pages={5823-5836},

  doi={10.1109/TAC.2022.3225758}}

@INPROCEEDINGS{wang2021nonlinear,

  author={Wang, Miaomiao and Tayebi, Abdelhamid},

  booktitle={60th IEEE Conference on Decision and Control (CDC)}, 

  title={Nonlinear Attitude Estimation Using Intermittent Linear Velocity and Vector Measurements}, 

  year={2021},

  volume={},

  number={},

  pages={4707-4712},

  doi={10.1109/CDC45484.2021.9683280}}

@article{oliveira2024pitot,
title = {Pitot Tube Measure-Aided Air Velocity and Attitude Estimation in {GNSS} Denied Environment},
journal = {European Journal of Control},
volume = {80},
pages = {101070},
year = {2024},
note = {2024 European Control Conference Special Issue},
issn = {0947-3580},
author = {Tomas L. {de Oliveira} and Pieter {van Goor} and Tarek Hamel and Robert Mahony and Claude Samson}
}

@ARTICLE{hua2013implementation,

  author={Hua, Minh-Duc and Ducard, Guillaume and Hamel, Tarek and Mahony, Robert and Rudin, Konrad},

  journal={IEEE Transactions on Control Systems Technology}, 

  title={Implementation of a Nonlinear Attitude Estimator for Aerial Robotic Vehicles}, 

  year={2014},

  volume={22},

  number={1},

  pages={201-213},

  doi={10.1109/TCST.2013.2251635}}

@ARTICLE{Bryne2017,

  author={Bryne, Torleiv H. and Hansen, Jakob M. and Rogne, Robert H. and Sokolova, Nadezda and Fossen, Thor I. and Johansen, Tor A.},

  journal={IEEE Control Systems Magazine}, 

  title={Nonlinear Observers for Integrated {INS$\backslash$/GNSS} Navigation: Implementation Aspects}, 

  year={2017},

  volume={37},

  number={3},

  pages={59-86},

  doi={10.1109/MCS.2017.2674458}}

@Inbook{Besancon2007,
author="Besan{\c{c}}on, Gildas",
title={An overview on observer tools for nonlinear systems},
bookTitle="Nonlinear Observers and Applications",
year="2007",
publisher="Springer Berlin Heidelberg",
address="Berlin, Heidelberg",
pages="1--33",
isbn="978-3-540-73503-8"
}

@InProceedings{Pascal2017,
author="Morin, Pascal
and Eudes, Alexandre
and Scandaroli, Glauco",
editor="Nielsen, Frank
and Barbaresco, Fr{\'e}d{\'e}ric",
title={Uniform Observability of Linear Time-Varying Systems and Application to Robotics Problems},
booktitle="Geometric Science of Information",
year="2017",
publisher="Springer International Publishing",
address="Cham",
pages="336--344",
}

@book{Ma2004rodriguesformula,
  author    = {Yi Ma and Stefano Soatto and Jana Ko{\v{s}}eck{\'a} and S. Shankar Sastry},
  title     = {An Invitation to {3-D} Vision: From Images to Geometric Models},
  publisher = {Springer},
  address   = {New York},
  year      = {2004},
  volume    = {26},
  series    = {Interdisciplinary Applied Mathematics},
  pages     = {27--29},
  doi       = {10.1007/b97515},
  isbn      = {978-0-387-00893-6},
  note      = {Rodrigues' formula derivation in Chapter 2}
}

@article{mahony2008nonlinear,
  author    = {Robert Mahony and Tarek Hamel and Jean-Michel Pflimlin},
  title     = {Nonlinear Complementary Filters on the Special Orthogonal Group},
  journal   = {IEEE Transactions on Automatic Control},
  volume    = {53},
  number    = {5},
  pages     = {1203--1218},
  year      = {2008},
  month     = {June},
  doi       = {10.1109/TAC.2008.923738},
  issn      = {1558-2523}
}

@article{VanGoor2025,
  author  = {Pieter Van Goor and Tarek Hamel and Robert Mahony},
  title   = {Synchronous Observer Design for Inertial Navigation Systems with Almost-Global Convergence},
  journal = {Automatica},
  year    = {2025},
  volume  = {177},
  pages   = {112328},
  doi     = {10.1016/j.automatica.2024.112328}
}

@article{Angeli2010,
  author  = {D. Angeli and L. Praly},
  title   = {Stability Robustness in the Presence of Exponentially Unstable Isolated Equilibria},
  journal = {IEEE Transactions on Automatic Control},
  volume  = {56},
  number  = {7},
  pages   = {1582--1592},
  year    = {2010},
  doi     = {10.1109/TAC.2010.2048471},
  issn    = {0018-9286},
  publisher = {IEEE}
}

@article{Angeli2004,
  author    = {D. Angeli},
  title     = {An Almost Global Notion of Input-to-State Stability},
  journal   = {IEEE Transactions on Automatic Control},
  volume    = {49},
  number    = {6},
  pages     = {866--874},
  year      = {2004},
  doi       = {10.1109/TAC.2004.829603},
  issn      = {0018-9286},
  publisher = {IEEE}
}

@unpublished{tchonkeu2025barometer,
    author={Méloné Nyoba Tchonkeu and Soulaimane Berkane and Tarek Hamel},
    title={Barometer-Aided Attitude Estimation}, 
    year={2025},
    eprint={2509.13649},
    archivePrefix={arXiv},
    primaryClass={cs.RO},
    url={https://arxiv.org/abs/2509.13649},
    note = {Accepted at 2026 IEEE American Control Conference (ACC)}
}

@unpublished{tchonkeu2026pitot,
    author={Melone Nyoba Tchonkeu and Soulaimane Berkane and Tarek Hamel},
    title={Pitot-Aided Attitude and Air Velocity Estimation with Almost Global Asymptotic Stability Guarantees}, 
    year={2026},
    eprint={2512.06133},
    archivePrefix={arXiv},
    primaryClass={eess.SY},
    url={https://doi.org/10.48550/arXiv.2602.08273},
    note = {Accepted at IEEE Conference on Control Technology and Applications (CCTA) 2026}
}

@inproceedings{meier2015px4,
  title={{PX4}: A Node-Based Multithreaded Open Source Robotics Framework for Deeply Embedded Platforms},
  author={Meier, Lorenz and Honegger, Dominik and Pollefeys, Marc},
  booktitle={2015 IEEE International Conference on Robotics and Automation (ICRA)},
  pages={6235--6240},
  year={2015},
  organization={IEEE}
}

@misc{px4_ecl_ekf_tuning,
    title ={Using {PX4}'s Navigation Filter {(EKF2)} - {PX4} User Guide (Main)},
    url={https://docs.px4.io/main/en/advanced_config/tuning_the_ecl_ekf},
    note = {Accessed: 2026-07-08}
}

@article{grip2011attitude,
  title={Attitude estimation using biased gyro and vector measurements with time-varying reference vectors},
  author={Grip, H{\aa}vard Fj{\ae}r and Fossen, Thor I and Johansen, Tor A and Saberi, Ali},
  journal={IEEE Transactions on automatic control},
  volume={57},
  number={5},
  pages={1332--1338},
  year={2011},
  publisher={IEEE}
}

@article{hansen2017nonlinear,
  title={Nonlinear observer design for GNSS-aided inertial navigation systems with time-delayed GNSS measurements},
  author={Hansen, Johan and Fossen, Thor I. and Johansen, Tor A.},
  journal={Control Engineering Practice},
  volume={61},
  pages={188--197},
  year={2017},
  publisher={Elsevier}
}

@inproceedings{bouazza2025observer,
  author={Bouazza, Tarek and Berkane, Soulaimane and Hua, Minh-Duc and Hamel, Tarek},
  booktitle={2025 IEEE 64th Conference on Decision and Control (CDC)}, 
  title={Observer Design for Optical Flow-Based Visual-Inertial Odometry with Almost-Global Convergence}, 
  year={2025},
  volume={},
  number={},
  pages={7419-7424},
  doi={10.1109/CDC57313.2025.11312849}}

@inproceedings{mourikis2007multi,
  title={A multi-state constraint Kalman filter for vision-aided inertial navigation},
  author={Mourikis, Anastasios I and Roumeliotis, Stergios I},
  booktitle={Proceedings 2007 IEEE international conference on robotics and automation},
  pages={3565--3572},
  year={2007},
  organization={IEEE}
}

@inproceedings{bloesch2015robust,
  title={Robust visual inertial odometry using a direct EKF-based approach},
  author={Bloesch, Michael and Omari, Sammy and Hutter, Marco and Siegwart, Roland},
  booktitle={2015 IEEE/RSJ international conference on intelligent robots and systems (IROS)},
  pages={298--304},
  year={2015},
  organization={IEEE}
}

@inproceedings{koditschek1988application,
  title={Application of a new Lyapunov function to global adaptive attitude tracking},
  author={Koditschek, Daniel E},
  booktitle={Proceedings of the 27th IEEE Conference on Decision and Control},
  pages={63--68},
  year={1988},
  organization={IEEE}
}

\begin{IEEEbiography}[{\includegraphics[width=1in,height=1.25in,clip,keepaspectratio]{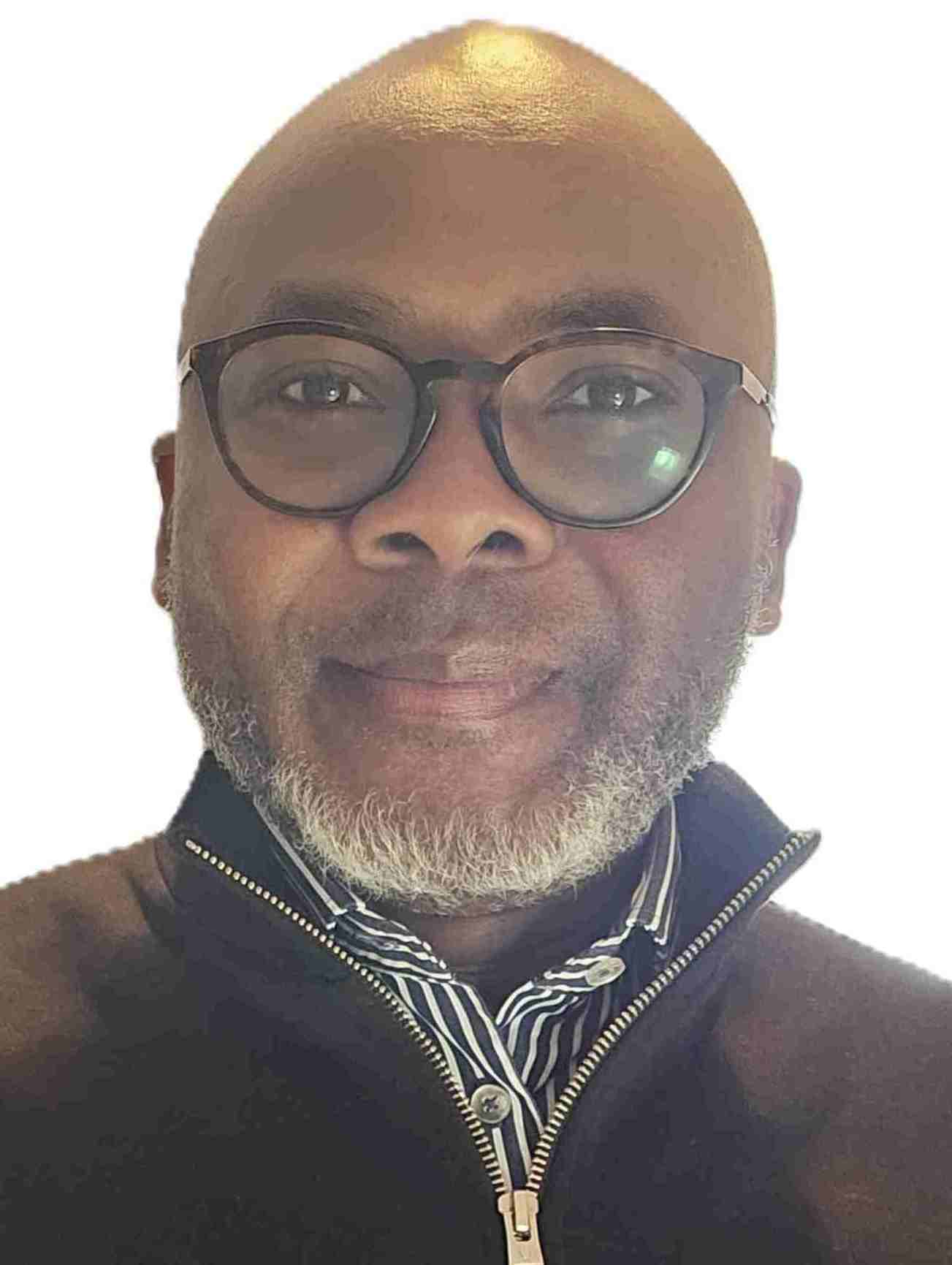}}]
{Méloné Nyoba Tchonkeu}
received his Engineer degree in Electrical Engineering from the University of Applied Sciences Western Switzerland, Switzerland, in 2008,  his M.Eng. degree in Aerospace Engineering (Avionics and Control Systems) from Polytechnique Montreal, Canada, in 2016, and his M.A.Sc. degree in Electrical Engineering from the University of Quebec in Outaouais, Canada, in 2025. He is currently pursing his Ph.D. degree in Science and Information Technology at the University of Quebec in Outaouais.
Concurrently with his academic training, he has held several project and systems engineering positions in the aerospace, automotive, and public sectors. His professional experience includes roles with the Canadian Space Agency and the Department of National Defence, where he currently serves as a Program Lead. 
He is a licensed Professional Engineer (P.Eng.) in the Province of Québec. His research interests are in the areas of nonlinear control theory with applications to unmanned robotic and intelligent autonomous systems.   
\end{IEEEbiography}

\begin{IEEEbiography}[{\includegraphics[width=1in,height=1.25in,clip,keepaspectratio]{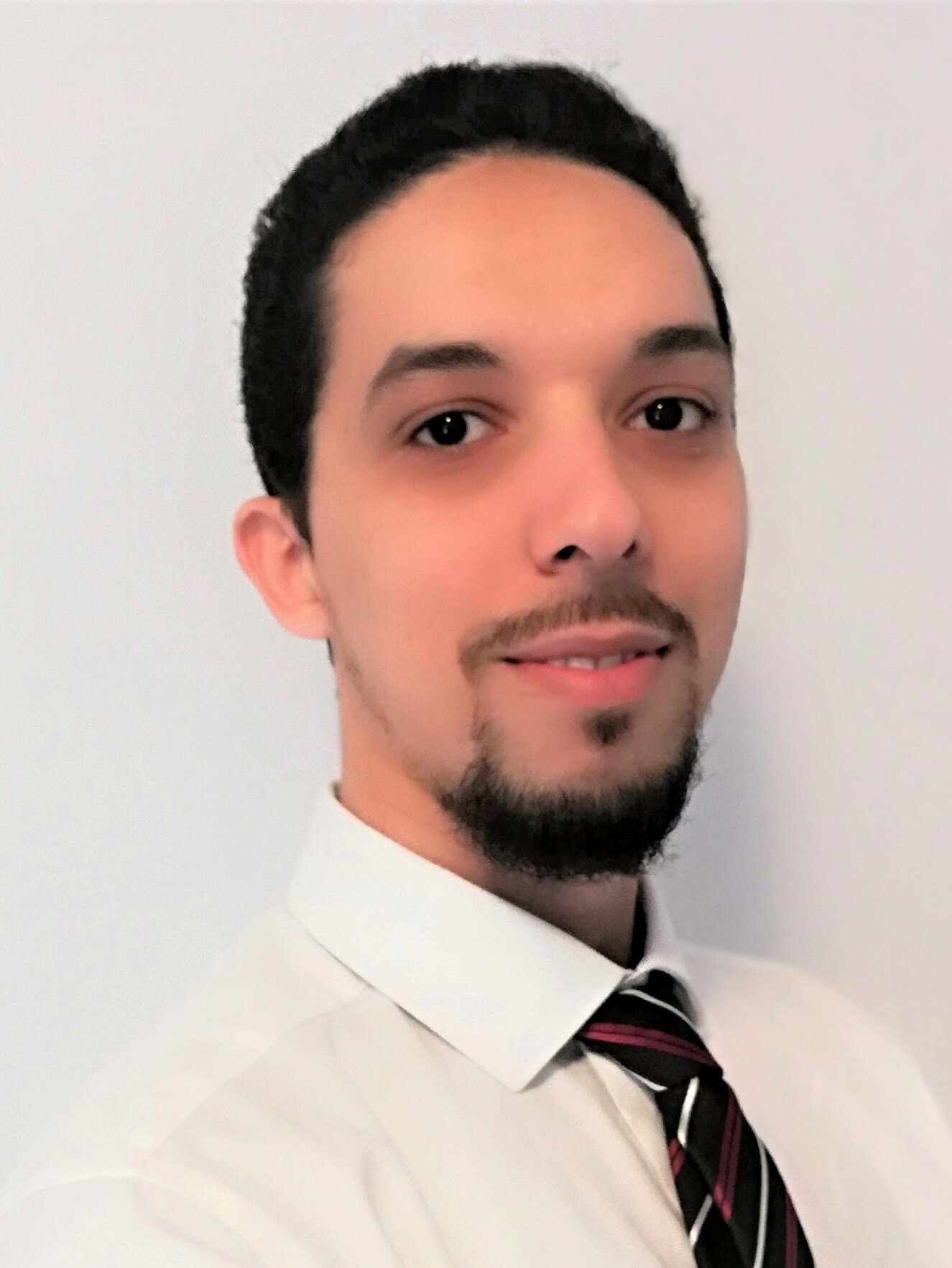}}]
{Soulaimane Berkane}
received his Engineering and M.Sc. degrees in Automatic Control from Ecole Nationale Polytechnique, Algeria, in 2013, and his PhD in Electrical Engineering from the University of Western Ontario, Canada, in 2017. He held postdoctoral positions at the University of Western Ontario, Canada, and at KTH Royal Institute of Technology, Sweden, between 2018 and 2019. 
He is currently an Associate Professor at the Department of Computer Science and Engineering, University of Quebec in Outaouais, Canada. He is a Senior Member of IEEE and a Professional Engineer (P.Eng.) in Ontario. He serves as an Associate Editor for the IEEE CSS Conference Editorial Board. His research interests are in the area of nonlinear control theory with applications to robotics and autonomous systems.
\end{IEEEbiography}

\begin{IEEEbiography}[{\includegraphics[width=1in,height=1.25in,clip,keepaspectratio]{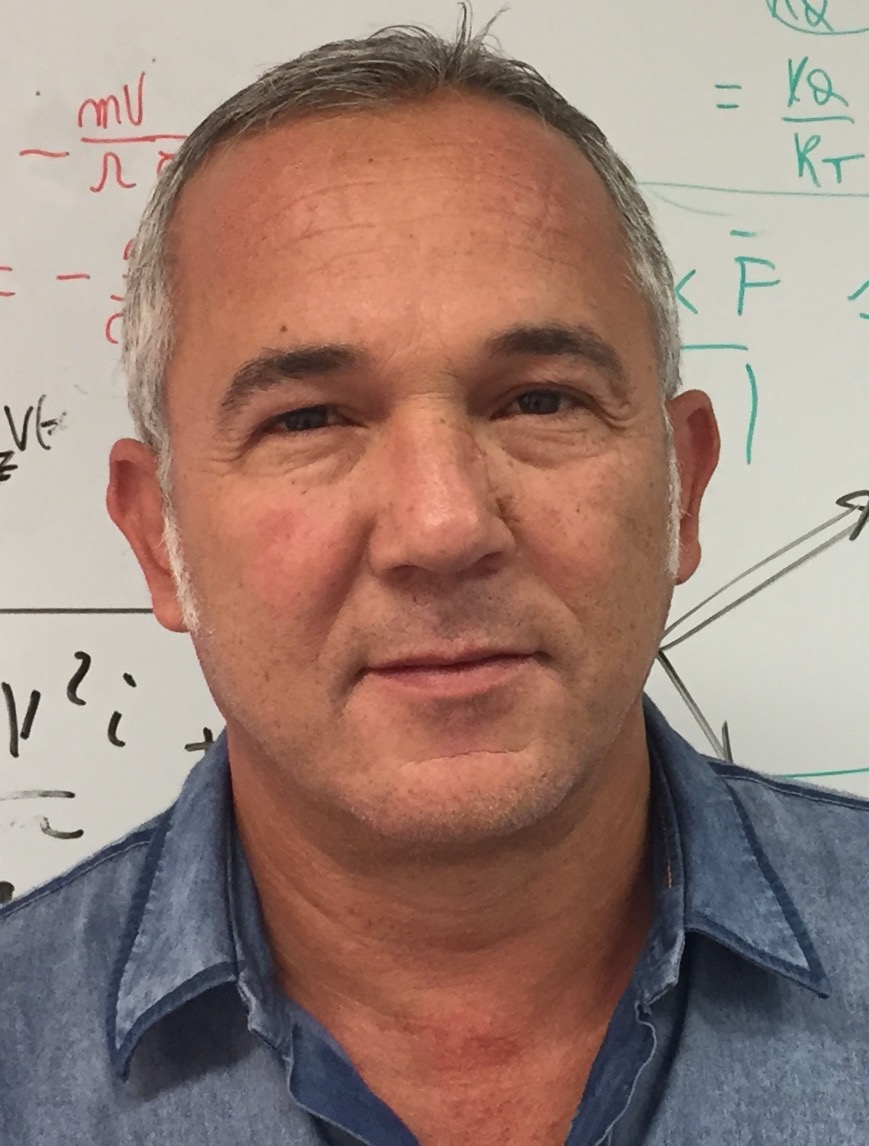}}]
{Tarek Hamel}
has been a Professor at the University Côte d'Azur since 2003. He received his Ph.D. in Robotics from the University of Technology of Compiègne (UTC), France, in 1996. After two years as a research assistant at UTC, he joined the Centre d'Études de Mécanique
d'Île-de-France in 1997 as an Associate Professor. His research
interests encompass nonlinear control theory, estimation, and
vision-based control, with a particular focus on applications to
unmanned robotic systems. Prof. HAMEL is an IEEE Fellow and a senior
member of the Institut Universitaire de France. He has served as an
Associate Editor for IEEE Transactions on Robotics, IEEE Transactions on Control Systems Technology, and Control Engineering Practice.
\end{IEEEbiography}

\vfill

\end{document}